\documentclass[10pt,journal,compsoc]{IEEEtran}

\ifCLASSOPTIONcompsoc
  \usepackage[nocompress]{cite}
\else
  \usepackage{cite}
\fi

\ifCLASSINFOpdf
   \usepackage[pdftex]{graphicx}
\else
\fi

\usepackage{csquotes}
\usepackage{amsmath}
\usepackage{amsthm}
\newtheorem{theorem}{Theorem}
\theoremstyle{definition}
\newtheorem{definition}{Definition}[section]
\newtheorem{lemma}[theorem]{Lemma}
\usepackage{wrapfig}
\usepackage{cases}
\usepackage[ruled,linesnumbered]{algorithm2e}
\usepackage{multirow}
\usepackage{booktabs}
\usepackage{subcaption}
\usepackage{color}
\usepackage{amsfonts}

\begin{document}

%
\title{
InFi: End-to-End Learning to Filter Input for Resource-Efficiency in Mobile-Centric Inference
~\thanks{This article is a substantially extended and revised version of Yuan et al.~\cite{infi}, which appeared in the proceedings of the 28th Annual International Conference on Mobile Computing And Networking (ACM MobiCom '22).}
}

%
%
%
%

\author{Mu~Yuan,~
        Lan~Zhang,~\IEEEmembership{Member,~IEEE,}
        Fengxiang~He,~\IEEEmembership{Member,~IEEE,}
        Xueting~Tong,~
        Miao-Hui~Song,~
        Zhengyuan~Xu,~\IEEEmembership{Senior Member,~IEEE,}
        and~Xiang-Yang~Li,~\IEEEmembership{Fellow,~IEEE}
\IEEEcompsocitemizethanks{
\IEEEcompsocthanksitem Lan Zhang is the corresponding author.
\IEEEcompsocthanksitem Lan Zhang is with the School of Computer Science and Technology and School of Data Science, University of Science and Technology of China, Hefei, China, and Institute of Artificial Intelligence, Hefei Comprehensive National Science Center, Hefei, China.\protect\\
E-mail: zhanglan@ustc.edu.cn
\IEEEcompsocthanksitem Mu Yuan, Miao-Hui Song and Xiang-Yang Li are
with the School of Computer Science and Technology, University of Science and Technology of China, Hefei, China.\protect\\
E-mail: ym0813@mail.ustc.edu.cn, songmiaohui@mail.ustc.edu.cn, xiangyangli@ustc.edu.cn
\IEEEcompsocthanksitem Xueting Tong is with the Institute of Advanced Technology, University of Science and Technology of China, Hefei, China.\protect\\
E-mail: tongxueting@mail.ustc.edu.cn
\IEEEcompsocthanksitem Fengxiang He is with JD Explore Academy, JD.com Inc., Beijing, China.\protect\\
E-mail: fengxiang.f.he@gmail.com
\IEEEcompsocthanksitem Zhengyuan Xu is with the Key Laboratory of Wireless Optical Communications, Chinese Academy of Sciences, University of
Science and Technology of China, Hefei, China.\protect\\
E-mail: xuzy@ustc.edu.cn
}
}

\IEEEtitleabstractindextext{%
\begin{abstract}
Mobile-centric AI applications have high requirements for the resource-efficiency of model inference.
Input filtering is a promising approach to eliminate redundancy so as to reduce the cost of inference.
Previous efforts have tailored effective solutions for many applications, but left two essential questions unanswered: 
(1) \emph{theoretical filterability of an inference workload} to guide the application of input filtering techniques, thereby avoiding the trial-and-error cost for resource-constrained mobile applications; 
(2) \emph{robust discriminability of feature embedding} to allow input filtering to be widely effective for diverse inference tasks and input content.
To answer them, we first formulate the input filtering problem and theoretically compare the hypothesis complexity of inference models and input filters to understand the optimization potential.
Then we propose the first end-to-end learnable input filtering framework that covers most state-of-the-art methods and surpasses them in feature embedding with robust discriminability.
We design and implement \textit{InFi} that supports different input modalities and mobile-centric deployments.
Comprehensive evaluations confirm our theoretical results and show that \textit{InFi} outperforms strong baselines in applicability, accuracy, and efficiency.
\textit{InFi} can achieve 8.5$\times$ throughput and save 95\% bandwidth, while keeping over 90\% accuracy, for a video analytics application on mobile platforms.
\end{abstract}

}

\maketitle

\IEEEdisplaynontitleabstractindextext

%
\IEEEpeerreviewmaketitle

\IEEEraisesectionheading{\section{Introduction}\label{sec:intro}
}

\IEEEPARstart{T}{he} increased computing power of mobile devices and the growing demand for real-time sensor data analytics have created a trend of mobile-centric artificial intelligence (AI)~\cite{lasagna, on-device-inference-survey, edge-ai, in-edge-ai}.
It is estimated that over 80\% of  enterprise IoT projects will incorporate AI by 2022.
The on-device inference of computer vision models brings us increasingly rich real-time AR applications on mobile devices~\cite{mobile-ar}.
A judicious combination of on-device and edge computing can analyze videos taken by drones in real-time~\cite{drone-video}.
The resource efficiency of model inference is critical for AI applications, especially for resource-limited mobile devices and latency-sensitive tasks.
However, many AI models with state-of-the-art accuracy~\cite{sota-seg, sota-pose, sota-nlp} are too computationally intensive to perform high-throughput inference, even when they are offloaded to edge or cloud servers~\cite{elf}.

For resource-efficient inference, one direct and popular way is to eliminate the redundancy of the deep model itself via accelerating and compressing techniques ~\cite{yolov3-tiny, light-openpose, mobile-bert, mobilenet-v2, mnasnet, mcdnn, automc}.
In this work, we follow another series of approaches~\cite{ff, reducto, foggycache, potluck, glimpse, noscope, infi} that attempt to filter the redundancy in the input data.
Fig.~\ref{fig:intro} shows four examples of input redundancy in mobile-centric AI applications.
We call this series of approaches input filtering and classify them into two categories: SKIP and REUSE.
(1) \textbf{SKIP} methods~\cite{ff, noscope} aim to filter input data that will bring useless inference results, e.g., images without faces for a face detector (Fig.~\ref{fig:intro1}) and audios without a valid command for a speech recognizer (Fig.~\ref{fig:intro2}). FilterForward~\cite{ff} trains a binary classifier and sets a threshold on classification confidence to filter input images.
(2) \textbf{REUSE} methods~\cite{foggycache, potluck} attempt to filter input whose results can reuse the previous inference results, e.g., motion signals of the same action (Fig.~\ref{fig:intro3}) and video frames with the same vehicle count (Fig.~\ref{fig:intro4}). 
FoggyCache~\cite{foggycache} maintains a cache of feature embedding and inference results of previous inputs and searches reusable results in the cache for newly arrived data.
Input filtering usually works as a necessary prelude to inference for under-resourced mobile systems.  
Moreover, compared with model optimizations, input filtering provides more flexible trade-offs between accuracy and efficiency, e.g., FilterForward can adjust the threshold in SKIP and FoggyCache can adjust the cache size in REUSE.
Although prior efforts have designed effective input filters for a range of applications, two important and challenging questions remain unanswered:

\begin{figure}[t]
     \centering
     \begin{subfigure}[b]{0.49\linewidth}
         \centering
         \includegraphics[width=\linewidth]{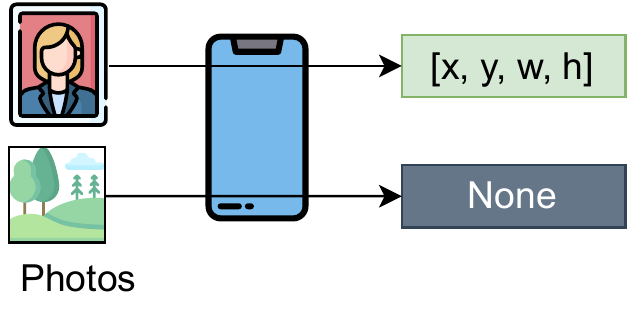}
         \caption{Face detection on mobile phones.}
         \label{fig:intro1}
     \end{subfigure}
     \hfill
     \begin{subfigure}[b]{0.49\linewidth}
         \centering
         \includegraphics[width=\linewidth]{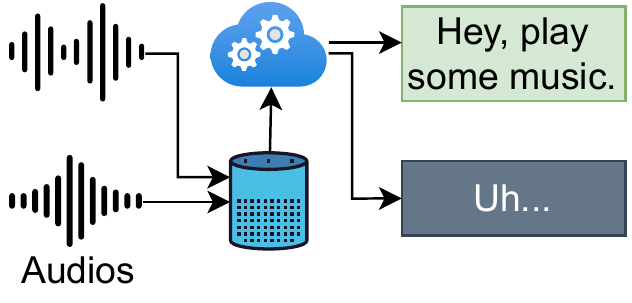}
         \caption{Speech recognition offloaded on the cloud.}
         \label{fig:intro2}
     \end{subfigure}
     \begin{subfigure}[b]{0.49\linewidth}
         \centering
         \includegraphics[width=\linewidth]{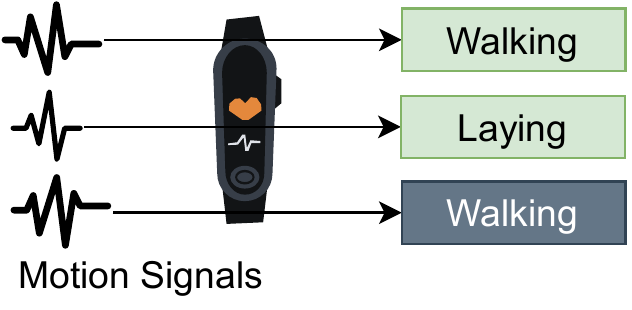}
         \caption{Human action recognition on smartbands.}
         \label{fig:intro3}
     \end{subfigure}
     \hfill
     \begin{subfigure}[b]{0.49\linewidth}
         \centering
         \includegraphics[width=\linewidth]{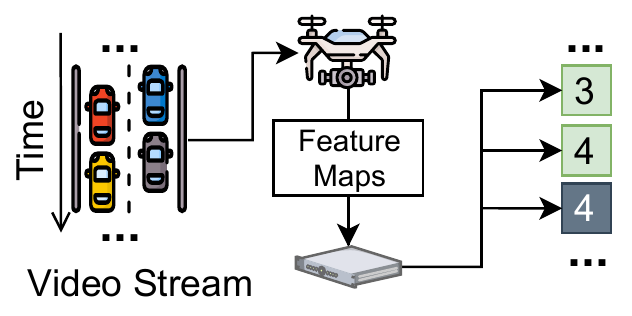}
         \caption{Vehicle counting using drones and edges.}
         \label{fig:intro4}
     \end{subfigure}
\caption{Input redundancy in mobile-centric AI applications. Gray squares indicate redundant inference results: (a) no detected face, (b) invalid recognized speech, (c) previous classification result can be reused, (d) latest count result can be reused.}
\label{fig:intro}
\end{figure}

\textbf{1. Theoretical filterability analysis for the guidance of applying input filtering to mobile-centric inference:} 
Not all inference workloads have the optimization potential by using input filtering.
Sometimes, to achieve the required accuracy, a SKIP/REUSE filter is more costly than the original inference.
Characterizing the conditions under which the filter has to cost more to be accurate is thus essential to input filtering.
Previous efforts study the input filtering problem from an application-oriented perspective.
They start from the observation of redundancy and propose bespoke input filtering solutions without further analyzing the relation between their inference workloads and input filters.
Without theoretical guidance and explanation, though they delivered accurate and lightweight input filters for specific workloads, the trial-and-error process of designing input filters for other workloads is still very cumbersome and may fail next time, especially for resource-scarce mobile systems.

\textbf{2. Robust feature discriminability for diverse tasks and modalities in mobile-centric inference:} 
A discriminative feature representation~\cite{dis-feat} is critical to filtering performance since it directly determines the accuracy of making SKIP decisions and finding REUSABLE results.
Recent work~\cite{reducto} shows that for different workloads, the discriminability of low-level features is different, e.g., the area feature works better for counting while the edge feature works better for detection.
Most existing filtering methods leverage handcrafted features~\cite{foggycache, potluck, reducto} or pre-trained neural networks as feature embedding ~\cite{ff}, and implicitly assume that these features are sufficiently discriminative for the target workloads.
However, mobile applications usually have high diversity in input content and inference tasks.
The dependency on pre-trained or handcrafted features leads to unguaranteed discriminability to these diversities.
Our experiments ($\S$~\ref{subsec:acc-r}) show that, for an action classification workload, neither a SKIP method using the pre-trained feature~\cite{ff} nor a REUSE method using the handcrafted feature~\cite{foggycache} can work effectively.
The feature embedding should be obtained in a workload-agnostic and learnable manner, rather than tailored case by case.

To answer these questions, we first provide a generic formulation of the input filtering problem and conditions of valid filters.
Then we theoretically define filterability and analyze the filterability of the two most common types of inference workloads (namely, classification and regression) by comparing the hypothesis complexity~\cite{foundationML, colt} of the inference model and its input filter.
Instead of designing bespoke solutions for narrowly-defined tasks, we propose the first end-to-end learnable framework which unifies both SKIP and REUSE approaches~\cite{ff, reducto, foggycache}.
The end-to-end learnability provides feature embedding with robust discriminability in a workload-agnostic manner, thus significantly broadening the applicability.
Based on the unified framework, we design an input filtering system, named \textit{InFi}, which supports both SKIP and REUSE functions.
In addition to image, audio, and video inputs, \textit{InFi} complements existing techniques in supporting text, sensor signal, and feature map inputs.
Previous methods are typically designed for a certain deployment, e.g., inference offloading~\cite{reducto, foggycache}.
\textit{InFi} flexibly supports common deployments in mobile systems, including on-device inference, offloading, and model partitioning~\cite{model-partition}.
In summary, our main contributions are as follows:

\noindent \textbullet We formulate the input filtering problem and provide validity conditions of a filter.
    We present the analysis based on complexity comparisons between hypothesis families of inference workloads and input filters, which can guide and explain the application of input filtering techniques.

\noindent \textbullet We propose the first end-to-end learnable input filtering framework that unifies SKIP and REUSE methods.
    Our framework covers most existing methods and surpasses them in feature embedding with robust discriminability, thus supporting more input modalities and inference tasks.

\noindent \textbullet We design and implement an input filtering system \textit{InFi}.
    Comprehensive evaluations on workloads with 8 input modalities, 14 inference tasks, and 3 types of mobile-centric deployments show that \textit{InFi} has wider applicability and outperforms strong baselines in accuracy and efficiency.
    For a video analytics application on a mobile platform, \textit{InFi} can achieve up to 8.5$\times$ throughput and save 95\% bandwidth compared with the naive vehicle counting workload, while keeping over 90\% accuracy.
\section{Input Filtering}
\label{sec:input-filtering}

This section formulates the input filtering problem and provides the conditions of a ``valid'' input filter for resource-efficient mobile-centric inference.

\subsection{Problem Definition}
\label{subsec:definition}

An input filtering problem needs to determine what input is redundant and should be filtered for a given inference model.
First, the definition of an input filtering problem is based on its target inference model.
Let $\mathcal{X}, \mathcal{Y}$ denote the input space and the label space of the target model, respectively.
Define $c: \mathcal{X}\rightarrow \mathcal{Y}$, named the target concept~\cite{pac}, which provides the ground-truth label for each input.
Then training a target model is to search for a function $h$ from a hypothesis family~\cite{pac} $\mathcal{H}$ using a set of training samples $S=\{(x_i, y_i)\}_{i=1}^m$, where $(x_1, ..., x_m)$ are sampled independently from $\mathcal{X}$ with an identical distribution $D$ and $y_i = c(x_i)$.
Using the above notations, we define the learning problem of the target inference model $h$ by $(\mathcal{X}, \mathcal{Y}, c, \mathcal{H}, D, S)$.
Step 0 in Fig.~\ref{fig:filtering} shows the original inference workflow of a trained model $h$, which takes input from $\mathcal{X}$ and returns an inference result $y \in \mathcal{Y}$.

Next, given a trained inference model $h$, its redundancy measurement function can be defined as:

\begin{definition}[Redundancy Measurement]
A redundancy measurement $f_h: \mathcal{Y} \rightarrow \mathcal{Z}$ of a model $h$ is a function that takes \textbf{only} the output of $h$ as input and returns a score that indicates whether the inference computation is redundant.
\end{definition}

Such measurements are common in practice.
For example, based on the output of a face detector the inference computation that returns no detected face is redundant and can be skipped, and we can set the score $z=0$;
Otherwise, $z=1$.
Formally, $y \mapsto 1(|y|>0)$, where $y$ is the output set of detected faces, $1(\cdot)$ is the indicator function.
For REUSE cases, if the inference result of an action classifier on a new query is the same as previously cached, the computation is redundant, and we can define $f_h(y) = 1(y \notin Y_{cached})$.
Note that, this definition of redundancy measurement does not depend on ground-truth labels, since our focus is not the accuracy but to optimize the resource efficiency of a deployment-ready target model with trusted accuracy by eliminating its redundant inference.
Step 1 in Fig.~\ref{fig:filtering} shows how redundancy measurement works.

Given the inference workload $h$ and redundancy measurement $f_h$, as Step 2 in Fig.~\ref{fig:filtering}, learning an input filter is defined as searching for a function $g$ from a hypothesis family $\mathcal{G}$ using a set of training samples $S'=\{(x_i, z_i)\}_{i=1}^n$, where $(x_1, ..., x_n)$ are sampled independently with distribution $D'$ and $z_i = f_h(h(x_i))$.
This learning problem is denoted by $(\mathcal{X}, \mathcal{Z}, f_h \circ h, \mathcal{G}, D', S')$, i.e., $g$'s target concept is the composite function of $f_h$ and $h$.

\begin{figure}[t]
    \centering
    \includegraphics[width=0.98\linewidth]{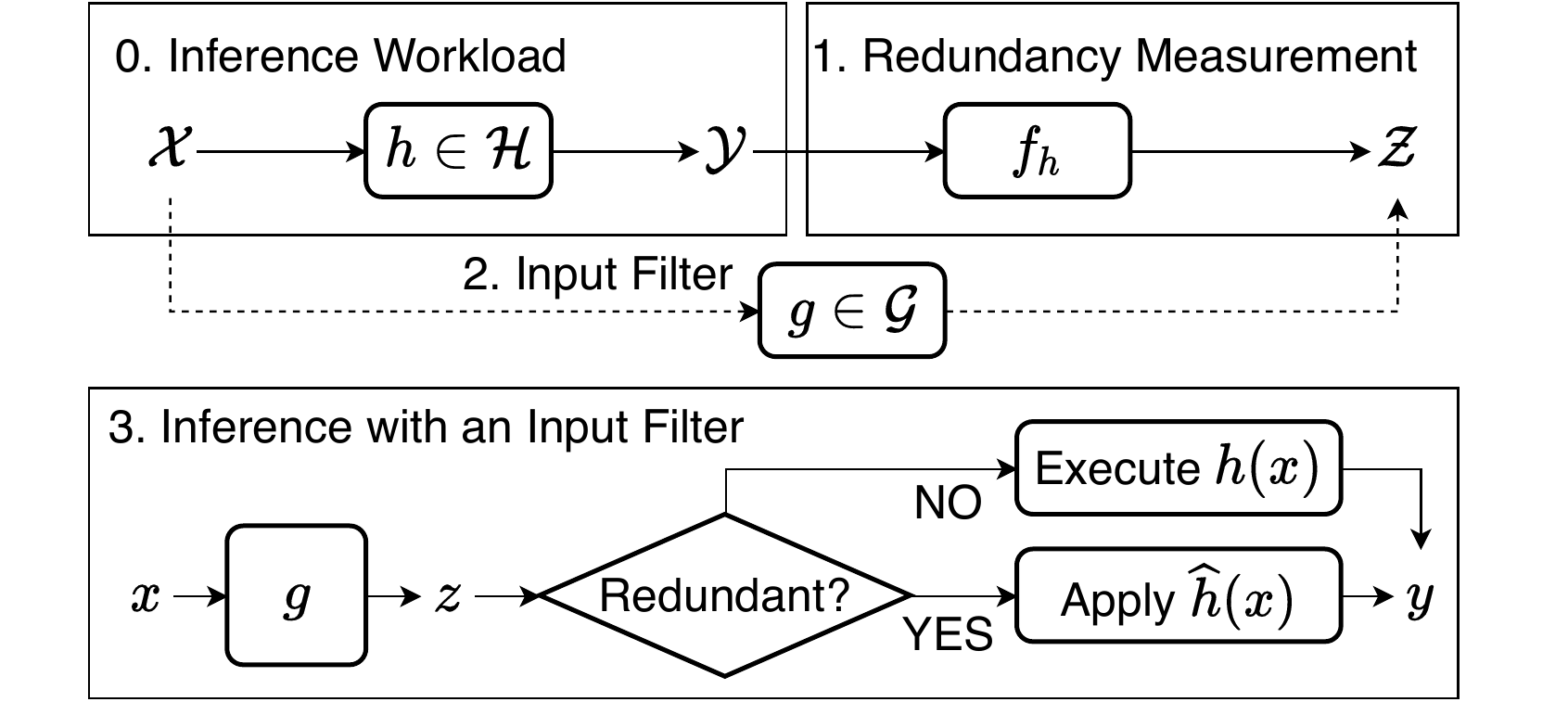}
    \caption{Overview of input filtering for inference workload.}
    \label{fig:filtering}
\end{figure}

\textbf{Inference with an input filter.}
Once an input filter $g$ is trained, the inference workflow changes from Step 0 to Step 3 in Fig.~\ref{fig:filtering}. The input filter $g$ becomes the entrance of the workload, which predicts the redundancy score $z$ of each input $x$.
If not redundant, the inference model $h$ will be directly executed on the input.

\subsection{Validity Conditions}
\label{subsec:validity}

After defining an input filter, we now give the conditions that a ``\textit{valid}'' input filter needs to meet for resource-efficient mobile inference.
The input filter is designed to balance the resource and accuracy: filtering more inputs can save more resources, but it also brings a higher risk of incorrect inference results. 

\textbf{Inference accuracy.}
With an input filter, the inference result $y$ for input $x$ is returned either by executing $h(x)$ or applying $\hat{h}(x)$.
Following previous work~\cite{foggycache,reducto,ff}, the correctness of the result $y$ refers to its consistency with the exact inference result by $h(x)$, rather than the ground-truth label.
An input filter's inference accuracy $Acc$ is defined as the ratio of correct results obtained by the inference workload with the filter.

\textbf{Filtering rate.}
The filtering rate, denoted by $r$, is defined as the ratio of filtered inputs (i.e., the ratio of results obtained by applying $\hat{h}$), which is also an important performance metric considered in previous work~\cite{ff,reducto,foggycache}.

\textbf{Overall cost.}
The overhead of an inference workload with an input filter needs to take $g$, $h$ and $\hat{h}$ into consideration.
Let $C(\cdot)$ denote the cost of a certain function.
For the cost of computation (e.g., runtime), the average cost per input changes from $C(h)$ into $C(g)+(1-r)C(h)+rC(\hat{h})$.
The communication cost (e.g., bandwidth) depends on the deployment of the mobile-centric inference workload.
On-device inference does not involve communication, while the overall bandwidth cost of offloading~\cite{ff,reducto} and model partitioning~\cite{model-partition} deployments becomes the original cost multiplied by $(1-r)<1$.

Based on the above metrics, we define an input filter as ``\textit{valid}'' if it satisfies two conditions:
1) \textbf{Accurate enough}: $Acc>T_{Acc}$, where $T_{Acc}$ is the threshold of acceptable inference accuracy.
2) \textbf{Reduced overhead}: the overall cost with an input filter is lower.
If we aim to reduce the computation cost, we need $(C(g)+(1-r)C(h)+rC(\hat{h}))/C(h) < 1$, i.e., $r>C(g)/(C(h)-C(\hat{h}))$;
If we aim to reduce the communication cost, we only need $r>0$.

\section{Filterability Analysis}
\label{sec:filterability}

As mentioned in Sec.~\ref{sec:intro}, not all inference workloads have the optimization potential by using input filtering techniques. 
Given an inference workload in a mobile-centric AI application, is there a valid input filter? 
To answer this question, based on our formulation of the input filtering problem, we first define the \textit{filterability} of an inference workload. 
Then we analyze filterability in three typical inference cases in SKIP settings and discuss uncovered cases.

\subsection{Definition of Filterability}
Given the learning problem $(\mathcal{X}, \mathcal{Y}, c, \mathcal{H}, D, S)$ of an inference model and the learning problem $(\mathcal{X}, \mathcal{Z}, f_h \circ h, \mathcal{G}, D', S')$ of its input filter, to simplify the analysis, we make assumptions as follows:
(1) $D=D'$, i.e., the training samples follow the identical distribution;
(2) $S'=\{(x_i, z_i)\}_{x_i \in S}$, i.e., the two learning problems share the same inputs in their training samples. 
But they are supervised under different labels.
The inference model $h$ is supervised by $y_i = c(x_i)$, while the input filter $g$ is supervised by $z_i = (f_h \circ h)(x_i)$.
Our intuitive idea for filterability is that, if an inference workload is filterable, the learning problem of its input filter should have lower complexity than the learning problem of its inference model.
Formally, we define \textit{filterability} as follows:
\begin{definition}[Filterability]
Let $Complex(\cdot)$ denote the complexity measurement of a hypothesis family.
We say that the inference workload is filterable, if $Complex(\mathcal{G})$ $\leq Complex(\mathcal{H})$, where $h\in \mathcal{H}$ and $(f_h \circ h) \in \mathcal{G}$.
\end{definition}
Since the hypothesis family cannot be determined based only on input and output spaces, we use the family of the input filter's target concept $f_h \circ h$ as $\mathcal{G}$.

Now we can characterize the theoretically achievable accuracy and overhead of the input filter for a given inference model by leveraging computational learning theory~\cite{foundationML}. 
It has been proven that the more complex the hypothesis family is, the worse the bounds of generalization error.
On the other hand, the hypothesis complexity of neural networks has a positive correlation with the number of parameters.
For example, let $W, L$ denote the number of weights and the number of layers in a deep neural network.
The VC-dimension~\cite{vcdim} (a measurement of the hypothesis complexity) is $O(WL\log(W)$~\cite{vcd-linear-nn}. 
In the case of the same layer structure, the more parameters the higher the inference overhead of neural networks. 
The generalization error bound and the number of parameters correspond to the accuracy and efficiency metrics in validity conditions ($\S$~\ref{subsec:validity}), respectively, although they are not strict quantification.
Therefore, if an inference workload is filterable, whose input filter has lower hypothesis complexity, we are confident to obtain a valid filter with sufficiently high accuracy and lower overhead than the inference model. 
Next, we will analyze the complexities of the hypothesis family of inference workload $h$ and its input filter $g$ in different cases.

\subsection{Low-Confidence Classification as Redundancy}
\label{subsec:conf}

Considering an inference workload, where the inference model is a binary classifier $h$ that returns the classification confidence, and the redundancy measurement regards the classification result with a confidence lower than a threshold $t$ as redundant, i.e., $f_h(y)=\text{sign}(y > t)$.
Confidence-based classification is very common in mobile AI applications, such as speaker verification.
We adopt the empirical Rademacher complexity~\cite{colt}, denoted by $\widehat{\mathfrak{R}}_S(\cdot)$, as the complexity measurement, which derives the following generalization bounds~\cite{foundationML}:
\begin{theorem}[Rademacher complexity bounds]
Let $\mathcal{H}$ be a family of hypothesis taking values in $\{-1,+1\}$.
Then for any $\delta > 0$, with probability at least $1-\delta$, the following holds for all $h\in \mathcal{H}$:
\begin{equation}
    R(h) \leq \widehat{R}(h) + \widehat{\mathfrak{R}}_S\mathcal{H}) + 3\sqrt{\frac{\log(2/\delta)}{2m}},
\end{equation}
where $R(h)$ and $\widehat{R}(h)$ denote the empirical and generalization errors, and $m$ is the number of training samples.
\end{theorem}
This theorem shows that the higher a hypothesis family's empirical Rademacher complexity, the worse the bounds of its generalization error.
The classification confidence-based redundancy measurement creates two hyperplanes parallel to $h=0$: points between them are considered redundant, and points outside them are considered not redundant.
Thus, the hypothesis family of the input filter's target concept has the form: $\mathcal{G}=\{\text{sign}( h(x) (h(x)+b))\}$, where $h \in \mathcal{H}$ and $b \in \mathbb{R}$.
Then we have proven the following lemma, which shows that the discussed inference workload is \textbf{not filterable}.

\begin{lemma}
\label{lemma:conf}
    Let $\mathcal{H}$ be a family of binary classifiers taking values in $\{-1, +1\}$.
    For $\mathcal{G}=\{\text{sign}(h(h+b))\}$ where $h\in \mathcal{H},b \in \mathbb{R}$:
    \begin{equation}
        \widehat{\mathfrak{R}}_S(\mathcal{G}) \geq \widehat{\mathfrak{R}}_S(\mathcal{H}).
    \end{equation}
\end{lemma}

\begin{proof}
    By definition,
    $$\widehat{\mathfrak{R}}_S(\mathcal{H})=E_\sigma[\sup_{h\in \mathcal{H}}(\frac{1}{m}\sum_{i=1}^m \sigma_i h(x_i))] $$ 
    and 
    $$\widehat{\mathfrak{R}}_S(\mathcal{G})=E_\sigma[\sup_{h\in \mathcal{H}, b\in \mathbb{R}}(\frac{1}{m}\sum_{i=1}^m \sigma_i \text{sign}(h(x_i)(h(x_i)+b))],$$
    where Rademacher variables $\sigma_i \in \{-1, +1\}$.
    Fixing $b=2$, 
    \begin{align*}
         \widehat{\mathfrak{R}}_S(\mathcal{G}) &\geq E_\sigma[\sup_{h\in \mathcal{H}, x_i \in S}(\frac{1}{m}\sum_{i=1}^m \sigma_i \text{sign}(h(x_i)(h(x_i)+2))]\\
         &= E_\sigma[\sup_{h\in \mathcal{H}, x_i \in S}(\frac{1}{m}\sum_{i=1}^m \sigma_i \text{sign}(h(x_i))]=\widehat{\mathfrak{R}}_S(\mathcal{H}),
    \end{align*}
    where we used the fact that $\text{sign}(h(x_i)+2)\equiv 1$.
\end{proof}

Multi-class classifiers can be treated as a set of confidence scoring functions, one for each class.
The above lemma can also be applied to derive that multi-class classifiers using such a confidence-based redundancy measurement are not filterable either.

\subsection{Class Subset as Redundancy}
\label{subsec:subset}

Considering the inference model $h$ as a multi-class mono-label classifier and $\mathcal{Y}=\{y_1, ..., y_l\}$.
Then its hypothesis family $\mathcal{H}$ has the form: $\mathcal{H}=\{\max(h_1, ..., h_l) : h_i \in \mathcal{H}_i, i\in [1,l]\}$, where $h_i$ returns the probability of the $i$-th class.
The redundancy measurement checks whether the predicted class belongs to a specific subset, i.e., $f_h(y)=1(y \in \mathcal{Y}')$, where $\mathcal{Y}'\subseteq \mathcal{Y}$.
It is common in mobile applications to select only a subset of labels for use.
For example, when deploying a pre-trained common object detector~\cite{mscoco} on a drone for traffic monitoring, we only care about the labels of vehicles and pedestrians, while considering other labels like animals and trees as redundancy.
With the class subset-based redundancy measurement, the hypothesis family of the input filter's target concept has the form: $\mathcal{G}=\{\max(h_i): y_i \in Y'\}$.
We have proven the following lemma, which shows that the discussed inference workload is \textbf{filterable}:
\begin{lemma}
Let $\mathcal{H}_1, ..., \mathcal{H}_l$ be $l$ hypothesis sets in $\mathbb{R}^\mathcal{X}$, $l\geq 1$, and let $\mathcal{H}=\{\max(h_1, ..., h_l) : h_i\in \mathcal{H}_i,i=1,...,l \}$.
For $\mathcal{G}=\{\max(h_i) : i \in J\}$, where $J\subseteq \{1,...,l\}$:
\begin{equation}
    \widehat{\mathfrak{R}}_S(\mathcal{G}) \leq \widehat{\mathfrak{R}}_S(\mathcal{H}).
\end{equation}
\end{lemma}

\begin{proof}
For any $j=1,...,l$:
\begin{align*}
    \widehat{\mathfrak{R}}_S(\mathcal{H}) &= \frac{1}{m} \mathop{E}_\sigma[\sup_{x_i \in S} \sigma_i \max_{h_k\in \mathcal{H}_k}(h_k(x_i))] \\
    &\geq \frac{1}{m} \mathop{E}_\sigma[\sup_{x_i \in S} \sigma_i \max_{j\in J}(h_j(x_i))]=\widehat{\mathfrak{R}}_S(\mathcal{G}).
\end{align*}
\end{proof}
The equation holds only if the max-value scoring function is in the selected subset for all $x_i \in S$, which means that without loss of inference accuracy, the optimal filterable ratio in the data is 0.
Except in this extreme case, we can think that the complexity of learning the input filter is strictly lower.

\subsection{Regression Bound as Redundancy}
\label{subsec:regression}

Considering a bounded regression model $h$, whose outputs are bounded by $M \in \mathbb{R}$ that $|h(x)-c(x)|\leq M$ (recall that $c$ is the target concept) for all $x \in X$.
The redundancy measurement checks whether the returned value is larger than a threshold, i.e., $f_h(y)=1(y > T)$.
As an example, face authentication on mobile devices usually requires the coordinates of the detected face to be within the specified range.
Then learning the target concept of the input filter becomes learning a regression model whose outputs are bounded by $T$, where $T < M$.
We also adopt the empirical Rademacher complexity and have the following theorem~\cite{foundationML}:
\begin{theorem}
Let $p\geq 1$ and $\mathcal{H}=\{x \mapsto |h(x)-c(x)|^p : h\in H\}$.
Assume that $|h(x)-c(x)| \leq M$ for all $x\in X$ and $h \in H$.
Then the following inequality holds:
  $ \widehat{\mathfrak{R}}_S(\mathcal{H}) \leq pM^{p-1}\widehat{\mathfrak{R}}_S(H)$.
\end{theorem}
Since $M>T$, this theorem shows that the upper bound of $\widehat{\mathfrak{R}}_S(\mathcal{G})$ is tighter than the upper bound of $ \widehat{\mathfrak{R}}_S(\mathcal{H})$.
So we can be confident that the bounded regression inference workload discussed is \textbf{filterable}.

\subsection{Discussions}

\textbf{Other inference tasks.}
Classification and regression are the most common inference tasks, and the three redundancy measurements discussed are widely adopted~\cite{adams,adams-tkdd,ff}.
However, there are some inference tasks that are challenging to measure the hypothesis complexity, like reinforcement learning~\cite{rl-rademacher} and structured learning~\cite{structured-complex}.
Besides, their redundancy measurements are typically ill-defined.
We believe that our problem formalization and analysis approach are general, based on which we will analyze the filterability of other tasks in future work.

\textbf{Characteristics of REUSE.}
For the REUSE approach, we cannot determine the hypothesis family of the input filter's target concept.
Here we only give one necessary condition: the inference result is discrete or can be discretized.
For example, classification and counting models return discrete outputs. 
But continuous localization coordinates of detection models cannot be reused directly unless reusing detection results with high IoU are regarded as correct, which is equivalent to discretizing the outputs.

\section{Framework}
\label{sec:framework}

In this section, we first propose a novel input filtering framework that unifies SKIP and REUSE approaches.
Then we discuss how existing approaches are covered by our framework and their limitations.
Finally, we present the key design, end-to-end learnability, and advantages it brings.

\subsection{SKIP as REUSE}

We unify SKIP and REUSE approaches based on the idea that:
\begin{displayquote}
    \textit{SKIP equals to REUSE the NONE output of $h(\vec{0})$.}
\end{displayquote}
Suppose we have an all-zero input $\vec{0}$ and apparently its inference result can be interpreted as NONE. 
Then given a new input $x$, if it is similar to $\vec{0}$ in the feature space, we can REUSE the cached NONE result, i.e., we SKIP the inference computation.
The key to reuse is to measure the semantic similarity between the current input and previously cached ones.
However, it is difficult to accurately measure semantic similarity directly based on the raw input.
As Step 1 in Fig.~\ref{fig:framework} illustrated, our framework first computes the feature embedding of each raw input.
Taking a pair of inputs $x, x'$, then our framework applies a difference function $d$ on their corresponding embeddings $e, e'$ and feeds the result into a classifier that predicts a single scalar $z$.
Under this framework, for SKIP, we fix $x'$ as an all-zero input $\vec{0}$, then the process degenerates to a binary classification task that takes $x$ as input and returns the prediction $z$.
In this way, our framework unifies SKIP and REUSE approaches, with only a difference in interpretation of the value $z$.
For REUSE, we interpret $z$ as the distance between two inputs.
For SKIP, we interpret $z$ as the probability that input $x$ is not redundant.

\begin{figure}[t]
    \centering
    \includegraphics[width=0.98\linewidth]{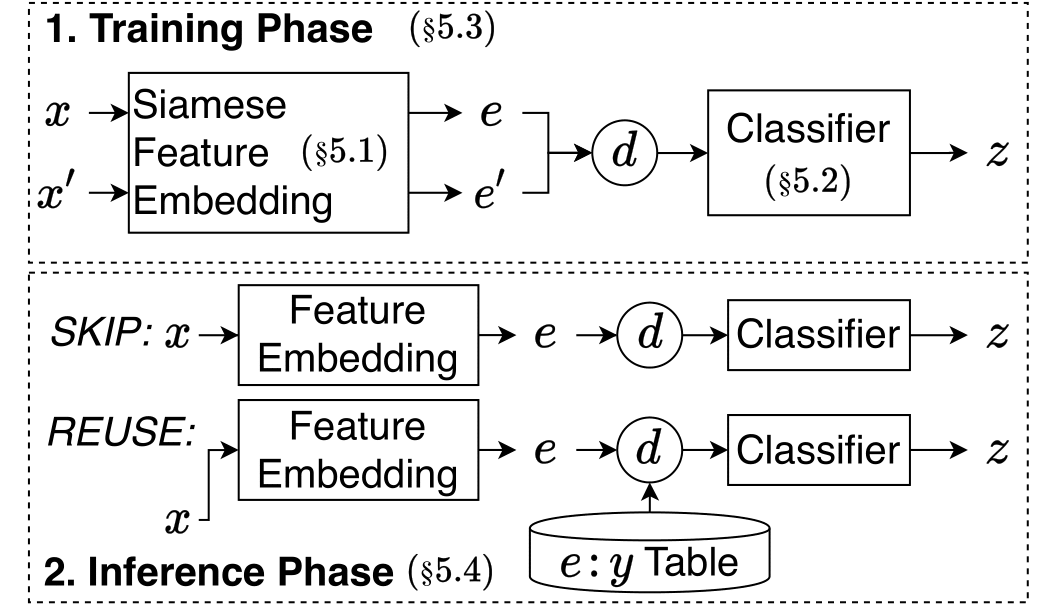}
    \caption{Unified and end-to-end learnable framework for both SKIP and REUSE input filtering.}
    \label{fig:framework}
\end{figure}

\subsection{Inference with an Input Filter}
\label{subsec:inference-with-filter}

For the inference phase, as shown in  Step 2 in Fig.~\ref{fig:framework}, SKIP and REUSE filters only differ in the inputs of the difference function $d$.
(1) \textbf{SKIP:}
Inference with a SKIP filter is the same as serving a binary classifier.
We can set a threshold on the predicted redundancy score $z$ to determine whether to skip.
(2) \textbf{REUSE:}
Inference with a REUSE filter needs to maintain a key-value table, where a key is a feature embedding and its value is the corresponding inference result.
For an arrived input $x$, the trained feature embedding network returns its embedding $e$ and the distances $z$ between $e$ and cached keys are computed by the difference function $d$ and the trained classifier.
Then we can leverage classification algorithms, e.g., KNN, to obtain the reusable cached results.

\subsection{Sub-Instance Approaches}
\label{subsec:sub-instance}

Here we explain how our framework covers three state-of-the-art input filtering methods~\cite{ff,reducto,foggycache} that will be used for comparison in our evaluations.

\textbf{Sub-instance1:}
FilterForward (FF)~\cite{ff} is a SKIP method for image input.
FF uses a pre-trained MobileNet's intermediate output as the feature embedding.
Then it trains a ``micro-classifier'' that consists of convolution blocks to make the binary decision for filtering.

\textbf{Sub-instance 2:}
FoggyCache (FC)~\cite{foggycache} is a REUSE method for image and audio input.
FC uses low-level features (SIFT for image, MFCC for audio) and applies locality-sensitive hashing (LSH) for embedding.
Then FC uses L2 norm as the difference function and applies KNN to get the reusable inference results from previously cached ones.

\textbf{Sub-instance 3:}
Reducto~\cite{reducto} is a variant of the SKIP method for video input.
It measures low-level features (pixel, edge, corner, area) differences between successive frames.
If they are similar enough, Reducto skips the current frame and returns the latest result.
Formally, let $x$ be the current frame and $x'$ be the previous frame.
Reducto defines $d(e, e')=(e-e')/e'$, where $e, e'$ are low-level features of $x,x'$.
It uses a threshold function as the classifier, i.e., $1(d(e,e') > T)$.

\subsection{End-to-end Learnability}
To obtain features with robust discriminability for diverse data modalities and inference tasks in mobile applications, a key design principle of our framework is end-to-end learnability.
End-to-end learning system casts complex processing components into coherent connections in deep neural networks~\cite{e2e-limit} and optimizes itself by applying gradient-based back-propagation algorithms all through the networks.
Deep end-to-end models have shown state-of-the-art performance on various tasks including autonomous driving~\cite{auto-drive} and speech recognition~\cite{deepspeech2}.
As aforementioned, a main component of our unified framework is to measure the semantic similarity between two inputs.
To make our framework end-to-end learnable, we leverage the metric learning paradigm, whose goal is to learn a task-specific distance function on two objects.
The metric learning paradigm turns the fixed difference function $d$ (e.g., Euclidean distance and L2 norm) used by existing methods into an end-to-end learnable network.
Within the metric learning paradigm, we adopt Siamese network structure~\cite{siamese} for feature embedding to support two inputs and flexible input modalities.
The Siamese network uses the same weights while working on two different inputs to compute comparable output vectors, and has been successfully applied in face verification~\cite{face-verification}, pedestrian tracking~\cite{pedestrian-track}, etc.
We can flexibly implement the Siamese feature embedding by incorporating different neural network blocks to learn modality-specific features in an end-to-end manner, instead of tailoring handcrafted or pre-trained feature modules.
Our experimental results show that the end-to-end learned features have robust discriminability to diverse inference workloads in mobile-centric AI applications.

\section{Design of InFi}
\label{sec:design}

Based on our input filtering framework, in this section, we present the concrete design of \textit{InFi} (INput FIlter), which supports both SKIP and REUSE functions, named \textit{InFi}-Skip and \textit{InFi}-Reuse.
The design of \textit{InFi} has four key components: feature embedding, classifier, training mechanism, and inference algorithm. We also discuss diverse deployments of \textit{InFi} in AI applications on mobile, edge, and cloud devices.

\subsection{Feature Networks for Diverse Input Modalities in Mobile-Centric AI}
\label{subsec:modality-net}
\textit{InFi} supports filtering inference workloads with six typical input modalities in mobile applications: text, image, video, audio, sensor signal, and feature map.
We develop a collection of modality-specific feature networks as building blocks for learning feature embedding.
Our major consideration in designing these feature networks is resource efficiency on mobile devices.

\textbf{Text modality ($g_{text}$).}
Text is tokenized into a sequence of integers, where each integer refers to the index of a token.
We adopt the word-embedding layer to map the sequence to a fixed-length vector by a transformation matrix and use a densely connected layer with a Sigmoid activation to learn the text features.

\textbf{Image modality ($g_{image}$).}
We use depth-wise separable convolution~\cite{separablecnn}, denoted by $SepConv$, to learn visual features.
$SepConv$ is a parameter-efficient and computation-efficient variant of the traditional convolution which performs a depth-wise spatial convolution on each feature channel separately and a point-wise convolution mixing all output channels.
Then we build residual convolution blocks~\cite{residual} $ConvRes$ as follows:
\begin{align*}
    &ConvStep(x) = LN(SepConv(ReLU(x))), \\
    &c_1(x) = ConvStep(x), c_2(x) = ConvStep(c_1(x)), \\
    &ConvRes(x) = MaxPool2D(c_2(x)) + ConvStep(x),
\end{align*}
where $ReLU$ denotes the rectified linear unit, $LN$ denotes the layer normalization and $MaxPool2D$ denotes the 2D max-pooling layer.
Finally, we build the image feature network with two residual blocks followed by a global max-pooling layer and a Sigmoid-activated dense layer.

\textbf{Video modality ($g_{video}$).}
For video modality, we need to represent not only the spatial but also the temporal features.
Given a window of frames, we stack one residual block for each frame and then concatenate their resulting feature maps.
Except for the first residual block, the video feature network performs the same operation as the image feature network.

\textbf{Audio modality ($g_{audio}$).}
We consider audio inputs in the form of either a 1D waveform or a 2D spectrogram and use the same structure as image feature networks to learn features from audio.

\textbf{Sensor signal and feature map modality ($g_{vec}$).}
Motion sensors are widely used in mobile devices and play a key role in many smart applications, e.g., gyroscope for augmented reality~\cite{gyroscope-ar} and accelerator for activity analysis~\cite{ucihar}.
Feature maps refer to the intermediate outputs of deep models and need to be transmitted in workloads that involve model partitioning~\cite{model-partition}.
We consider these two types of input as a vector with a fixed shape and use two densely connected layers to learn the feature embedding from the flattened vector.

\textbf{Flexible support for input modalities.}
Our design provides flexible support for diverse input modalities in mobile-centric AI applications.
We can easily integrate a modality-specific neural network from advanced machine learning research as the feature network block into our framework, so as to learn feature embeddings in an end-to-end way.

\subsection{Task-Agnostic Classifier}
Each feature network $g_{modality}$, where $modality$ belongs to \{text, image, video, audio, vec\},  takes $x$ as input and outputs the embedding $emb$. 
We add a dropout layer after the last dense layer of feature networks to reduce overfitting.
Following the previous design of Siamese network~\cite{siamese}, we use the absolute difference as the function $d$.
Let $emb_1, emb_2$ denote the embedding outputs of two inputs $x_1, x_2$.
The classifier is defined as $g_{cls} = \sigma(\sum_j w_j |emb_1^{(j)} - emb_2^{(j)}| + b)$, where $emb^{(j)}$ denotes the j-th element in the embedding vector and $\sigma$ is the Sigmoid function.
To sum up, the input filter function $g: \mathcal{X} \rightarrow \mathcal{Z}$ can be defined as $g(x) = (g_{cls} \circ g_{modality})(x)$.
With proper implementation, the modality of input data can be automatically detected without manually setting.

\begin{figure}[t!]
    \centering
    \includegraphics[width=0.95\linewidth]{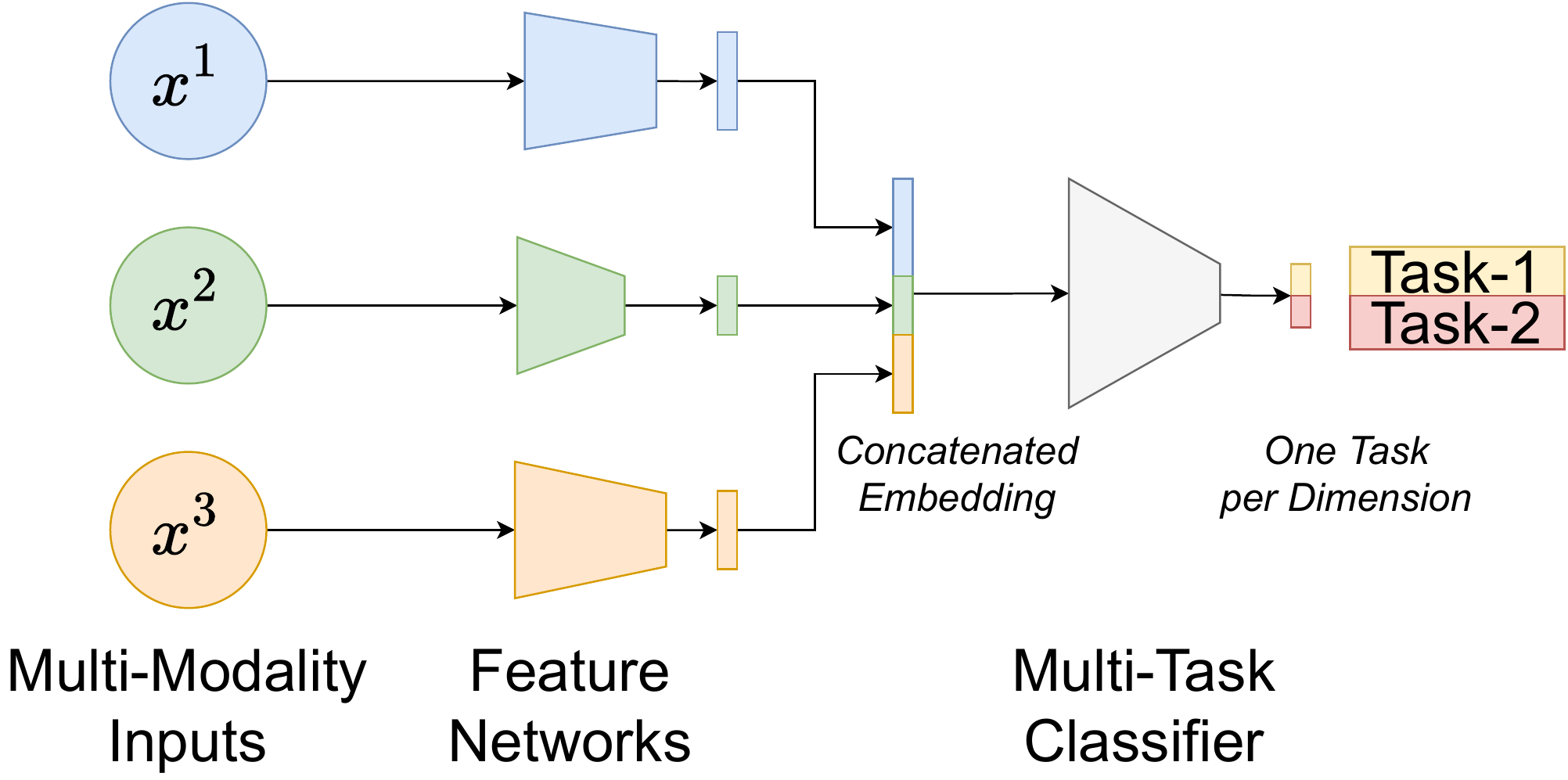}
    \caption{Extend input filtering to multi-modal and multi-task workloads. $x^1,x^2,x^3$ represents three input modalities.}
    \label{fig:mtask}
\end{figure}

\subsection{Multi-Task Extension}
\label{subsec:multitask}

The above design is described for single-task workloads, however, it is common to concurrently run multiple AI models in real applications.
We will show that the design of \textit{InFi} can be flexibly extended to multi-modal and multi-task inference workloads.

\textbf{Multi-modality single-task.}
Multi-modal learning aims to learn AI models given multiple inputs with different modalities, which is receiving increasing attention in areas such as autonomous driving~\cite{auto-drive}.
Our designs of modality-specific feature networks and task-agnostic classifier naturally support multi-modal extension: 
For each modality $mod \in$ \{text, image, video, audio, vec\}, we build the corresponding feature network $g_{mod}$ to learn its embedding.
Then we concatenate the resulting embeddings and feed it to the classifier $g_{cls}$.

\textbf{Single-modality multi-task.}
It is common to deploy multiple AI models to analyze the same input, e.g., detecting vehicles and classifying traffic conditions on the same video stream.
For input filtering, we simply extend the length of the last dense layer in the classifier $g_{cls}$, one dimension per task.
Existing work on multi-task learning~\cite{mtl} demonstrates that the cross-task representation improves learning performance.
Formally, given $t$ tasks, it has been proven that the sample complexity needed~\cite{theory-mtl} is as follows:
\begin{equation}
    Complex(g_{mod}) + t \cdot Complex(g_{cls}).
\end{equation}
That is, we can save $(t-1)\cdot Complex(g_{mod})$ sample complexity, compared with learning a filter for each task separately.
And our experimental results (Fig.~\ref{fig:mtl-infi}) also show that the cross-task representation is beneficial for input filtering.

\textbf{Multi-modality multi-task.}
Considering a general case where we need to filter inputs for multi-modality and multi-task workloads, we can combine the above two extensions, as shown in Fig.~\ref{fig:mtask}:
For each input modality, we build the corresponding feature network and concatenate the resulting embeddings;
And for each task, we build a multi-dimension classifier, one dimension per task, which takes the concatenated embedding as input.
Compared with the naive way that deploys independent \textit{InFi} for different inference workloads, our proposed extension saves computation and leverages potential advantages of cross-task and cross-modality representation.

\subsection{End-to-End Training}
\label{subsec:training}

\textit{InFi}-Skip and \textit{InFi}-Reuse share the same model architecture, but have different formats of training data.
1) Learning an \textit{InFi}-Skip filter uses the same paradigm as training a binary classifier.
Thus its training samples are ${(x_i, f_h(h(x_i)))}_{i=1}^n$ and we use the binary cross-entropy loss function.
In practice, we can use the original training set of $h$ or data collected during serving $h$.
Since $f_h$ only depends on the inference result, the supervision labels can be collected automatically.
2) \textit{InFi}-Reuse filters are trained using the contrastive loss~\cite{contrastive-loss} with a margin parameter of one.
Given a set of input and their discrete inference results, the redundancy measurement is defined as the distance metric between a pair of inputs.
Formally, a training sample consists of a pair of inputs and their distance label $(x_i, x_j, 1(y_i \neq y_j))$.
We can optimize all trainable parameters end-to-end, using standard back-propagation algorithms.

\textbf{Online active update.}
Unlike benchmark datasets, the distribution of real-world inputs, e.g., the video streams captured by surveillance cameras, is much narrower and changes online~\cite{online-kd}.
In a video-based vehicle counting application (see Sec.~\ref{subsec:config} for detailed setup), we explored the shifted distribution of frames over wall-clock time.
As shown in Fig.~\ref{fig:oneday-box}, the vehicle count varies with time.
There are two distinct count peaks in the morning and evening and the nighttime results remain stable at low values.
We noticed that the captured frames switched between infrared (IR) and RGB images when the lighting condition changed.
The RGB-IR technology provides day and night vision capability for cameras and is supported by popular commercial sensors. 
We split all frames into infrared and RGB subsets and plot their distribution over the number of detected vehicles in Fig.~\ref{fig:cdf-infrared}.
We can see a clear difference in the distribution.
Therefore, a vanilla offline training policy that selects initial samples (e.g., frames in the first hour) for training the input filter results in sub-optimal performance quickly.
To overcome the poor adaptability of offline training, we adopt the \textit{least confidence}~\cite{active} strategy to actively select samples for updates on the fly.
Existing work~\cite{active-complexity} proved that the sample complexity of active learning is asymptotically smaller than passive learning.
Specifically, we set a period length and a sampling ratio $\beta$\%.
Then we execute the input filter on all inputs within a period and select $\beta$\% samples with the least confidence ($|g(x)-0.5|$).
For \textit{InFi}-Reuse, we treat $1-\theta$ as the confidence score.
Our experimental results (Fig.~\ref{fig:online-active}) show that, given the same budget for the number of training samples, this active strategy significantly outperforms the offline one.

\begin{figure}[t]
     \centering
     \begin{subfigure}[b]{0.49\linewidth}
         \centering
         \includegraphics[width=\linewidth]{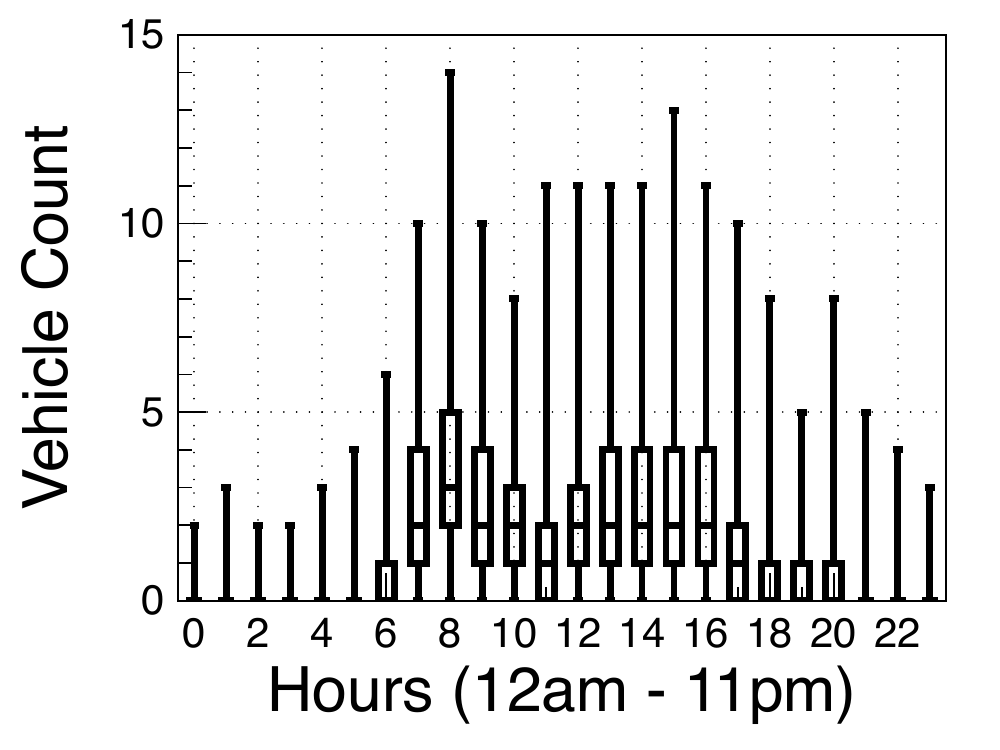}
         \caption{One-day trace of vehicle count.}
         \label{fig:oneday-box}
     \end{subfigure}
     \hfill
     \begin{subfigure}[b]{0.49\linewidth}
         \centering
         \includegraphics[width=\linewidth]{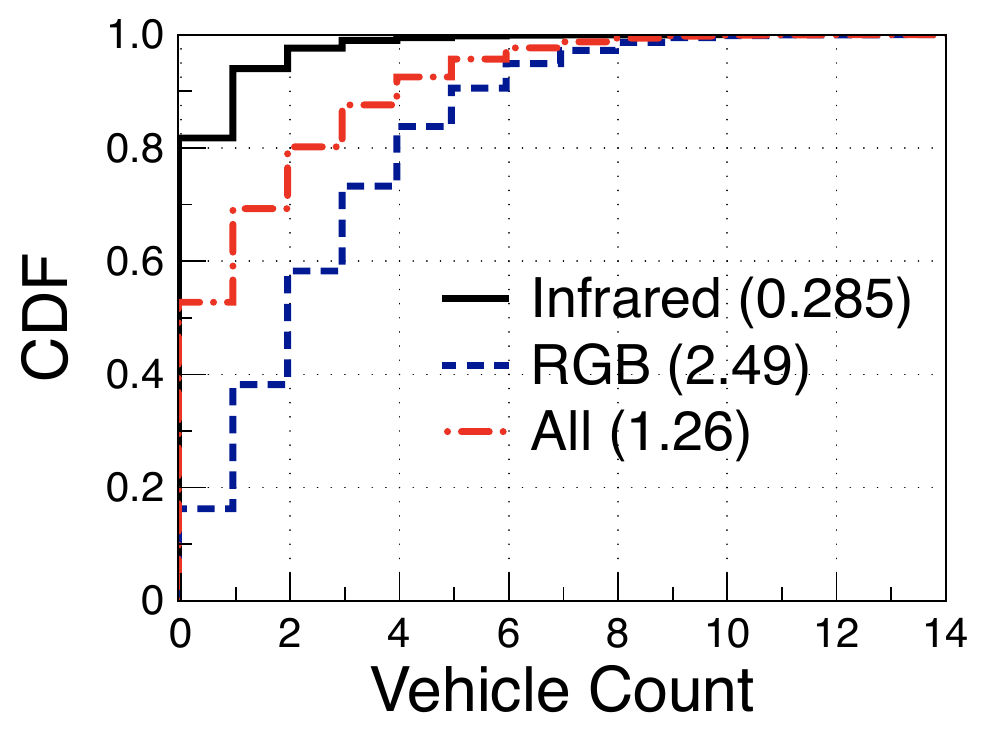}
         \caption{CDF of vehicle count on infrared and RGB frames.}
         \label{fig:cdf-infrared}
     \end{subfigure}
\caption{Distribution shifts in online video streams. (b) The means of the distributions are given inside parentheses.}
\label{fig:active-moti}
\end{figure}

\begin{algorithm}[t!]
\small
\SetKwInOut{Input}{input}
\SetKwProg{Def}{def}{:}{}
\SetKwFunction{KNN}{HKNN}
\SetKwFunction{read}{read}
\SetKwFunction{inference}{inference}
\SetKwFunction{Len}{Len}
\SetKwFunction{Replace}{replace}

    \Input{input source $src$, redundancy threshold $T$, cache size $s$, KNN parameter $K$, homogeneity threshold $\theta_T$}
    \Def{InFiSkip(src)}{
        \While{$x \leftarrow \read(src)$}{
            \uIf{$g(x) > T$}{
                $y \leftarrow \inference(x)$\;
            }
            \Else{$y \leftarrow None$\;}
        }
    }
    
    \Def{InFiReuse(src)}{
        Initialize empty $cache$\;
        \While{$x \leftarrow \read(src)$}{
            \uIf{$\Len(cache) < s$}{
                $y \leftarrow \inference(x)$\;
                $cache[g_{modality}(x)] \leftarrow y$\;
            }
            \Else{ 
                $y,\theta \leftarrow \KNN(cache, g_{modality}(x), g_{cls}, K)$\;
                \If{$\theta < \theta_T$}{
                    $y \leftarrow \inference(x)$\;
                    \Replace($cache$, \{$g_{modality}(x): y$\})\;
                }
            }
        }
    }
\caption{Inference with an \textit{InFi} Filter}
\label{alg:infi}
\end{algorithm}

\subsection{Inference Phase}

After training an \textit{InFi} filter, we integrate it into the original inference workload using Algorithm.~\ref{alg:infi}.

\textbf{\textit{InFi}-Skip.}
We set a redundancy threshold for \textit{InFi}-Skip to determine whether to skip the current input.
And if we skip the input, \textit{InFi}-Skip will return a NONE result, whose interpretation depends on the redundancy measurement in specific applications.
For example, NONE means no face detected in face detection, zero vehicles in vehicle counting application, meaningless speech in speech recognition, etc.

\textbf{\textit{InFi}-Reuse.}
To reuse previous inference results, we need to maintain a cache whose entry is a key-value pair of an input embedding and its inference results.
Following the previous RESUE approach~\cite{foggycache}, we adopt K-Nearest Neighbors (KNN) algorithm to reuse cached results.
But it is possible that a new input is not similar to any cached entries, i.e., a cache miss.
We adopt the Homogenized KNN (H-KNN)~\cite{foggycache} algorithm to handle this problem, which calculates a homogeneity score $\theta$ of the found K nearest neighbors and sets a threshold $\theta_T$ on the homogeneity score to detect the cache miss.
Then we can replace entries using policies like least frequently used (LFU), denoted by \texttt{replace} in Algorithm~\ref{alg:infi}.
Different from the original KNN that typically uses Euclidean distance, which is non-parametric, we set the distance measurement as the trained $g_{cls}$.
We denote $\text{HKNN}(cache, emb, g_{cls}, K)$ as the H-KNN function which returns the majority inference result $y$ of $K$ nearest neighbors of $emb$ in $cache.keys$ using the $g_{cls}$ to calculate the distance between embeddings, and computes $\theta$.
We focus on taking the advantage of end-to-end learnability, and other subtle optimization opportunities such as cache warm-up are out of the scope of this work.

\subsection{Mobile-Centric Deployments}
\label{subsec:deployment}

Unlike existing work tailored for specific deployment, e.g., inference offloading~\cite{ff,foggycache,reducto}, \textit{InFi} supports diverse mobile-centric deployments:
(1) \textbf{On-device:} both inference model and input filter are deployed on one device;
(2) \textbf{Offloading:} the input filter is deployed on one device, and the inference model is deployed on another device.
(3) \textbf{Model Partitioning (MP)}~\cite{model-partition}: the inference model is partitioned across two devices, and the input filter is deployed with the first part.
MP is a promising approach to collaboratively make use of the computing resources of mobile and edge devices~\cite{code-partition,dnn-partition} and better protect the privacy of mobile data~\cite{partition-privacy}.
For MP deployment, the filter's input is the feature map, so existing filtering approaches~\cite{ff,foggycache,reducto} cannot be applied.
Due to the support of feature map modality, \textit{InFi} is the first input filter that can be applied in model partitioning workloads.
Note that \textit{InFi} is not limited to systems with a single mobile and edge node.
For example, training one filter per server, or changing one filter's binary classifier into a multi-category one (one bit per server), \textit{InFi}-Skip can be used in the multi-tenancy context~\cite{ff}.
\section{Evaluation}
\label{sec:exp}

\subsection{Implementation and Configurations}
\label{subsec:config}
We implemented \textit{InFi}~\footnote{https://github.com/yuanmu97/infi} in Python.
We build all feature networks and classifiers with TensorFlow 2.4.
The learning rate is set as 0.001, the batch size is 32, and the number of training epochs is 20.
In the text feature network, the output dimension of the embedding layer is 32.
For image, video, and audio feature networks, we use 32 and 64 convolution kernels in the two residual blocks.
We use 128 units in the first dense layer in vector feature networks.
The last dense layer of all feature networks has 200 units and 0.5 dropout probability.

\begin{table}[t]
\centering
\caption{Datasets and Inference Workloads}
\label{tab:workload}
\begin{tabular}{@{}lll@{}}
\toprule
\textbf{Dataset} & \textbf{Modality} & \textbf{Inference Task} \\ \midrule
\multirow{7}{*}{Hollywood2} & Video Clip & Action Classification (AC) \\ \cmidrule(l){2-3} 
 & \multirow{3}{*}{Image} & Face Detection (FD) \\
 &  & Pose Estimation (PE) \\
 &  & Gender Classification (GC) \\ \cmidrule(l){2-3} 
 & Audio & Speech Recognition (SR) \\ \cmidrule(l){2-3} 
 & \multirow{2}{*}{Text} & Named Entity Recognition (NER) \\
 &  & Sentiment Classification (SC) \\ \midrule
ESC-10 & Audio & Anomaly Detection (AD) \\ \midrule
UCI HAR & Motion Signal & Activity Recognition (HAR) \\ \midrule
MoCap & Motion Signal & User Identification (UI) \\ \midrule
DeepSig & Radio Signal & Modulation Recognition (MR) \\ \midrule
WiFiHAR & WiFi CSI & Activity Recognition (WAR) \\ \midrule
\multirow{2}{*}{City Traffic} & Video Stream & Vehicle Counting (VC) \\ \cmidrule(l){2-3} 
 & Feature Map & Vehicle Counting (VC-MP) \\ \bottomrule
\end{tabular}
\end{table}

\textbf{Datasets and inference models.}
To evaluate \textit{InFi}'s wide applicability, we choose 14 inference workloads that cover 8 input modalities and three deployments (see Tab.~\ref{tab:workload}).
Seven datasets are used:
(1) We reprocessed a standard video dataset, Hollywood2~\cite{hw2}, to create four different input modalities: video clip, image, audio, and text.
An action classification model~\cite{actionmodel} is deployed on the original video clips.
Images are sampled from the video clips and a face detection~\cite{deepface}, a pose estimation~\cite{openpose}, and a gender classification~\cite{deepface} models are deployed.
Audio is extracted from each video clip and we deploy a speech recognition model~\cite{deepspeech2}.
Text is the caption generated on sampled images by an image captioning model~\cite{caption}.
A named entity recognition model (spaCy~\cite{spacy}) and a sentiment classification model~\cite{sentiment-model} are deployed.
(2) We use the ESC-10 dataset~\cite{esc} for audio anomaly detection and deploy an transformer-based model~\cite{ast-model}.
(3) We use the UCI HAR dataset~\cite{ucihar} for motion signal-based human activity recognition and deploy an LSTM-based model.
(4) We use the MoCap dataset~\cite{mocap} for training a motion signal-based user identification (12 users) model, using an LSTM-based architecture, and deploying it as the inference workload.
(5) We use the DeepSig dataset~\cite{dl-radio-clf} and deploy a ResNet-based model for modulation recognition of radio signals.
(6) We use the WiFiHAR dataset~\cite{wifi-dataset} for activity recognition and deploy an LSTM-based model.
(7) We collected a video dataset, named City Traffic, from a real city-scale video analytics platform.
We collected 48 hours of videos (1FPS) from 10 cameras at road intersections and deploy YOLOv3 re-implemented with TensorFlow 2.0 to count the number of vehicles in video frames.
All deployed inference models load publicly released pre-trained weights.
And we split each dataset for training and testing by 1:1 (Hollywood2 and UCI HAR are split randomly, while City Traffic is split by time on each camera).

\textbf{Devices and deployments.}
We use an edge server with one NVIDIA 2080Ti GPU and three mobile platforms: (1) NVIDIA JETSON TX2, (2) XIAOMI Mi 5, and (3) HUAWEI WATCH.
All device-independent metrics are tested on the edge.
For vehicle counting, we test three deployments: on-device, offloading, and model partitioning (see Sec.~\ref{subsec:deployment}).

\textbf{Baselines.}
We adopt three strong baselines: FilterForward (FF)~\cite{ff}, Reducto~\cite{reducto}, and FoggyCache (FC)~\cite{foggycache}.
See Sec.~\ref{subsec:sub-instance} for details of baselines.
For workloads with no existing method presented (to our best knowledge), we tested a method dubbed \textit{Low-level} that first computes low-level embedding for inputs (MFCC for audio, Bag-of-Words for text, raw data for motion signal and feature map).
Then \textit{Low-level} uses K-nearest neighbors vote (K=10) for both SKIP and REUSE cases.
We also deployed YOLOv3-tiny~\cite{yolov3-tiny} model for vehicle counting and a lightweight pose estimation model~\cite{light-openpose} to compare input filtering and model compression techniques.

\begin{table}[t]
\centering
\caption{Filtering rate (\%) @ 90\% inference accuracy of SKIP methods.}
\label{tab:skip-90acc}
\begin{tabular}{@{}llllllll@{}}
\toprule
Method & FD & PE & GC & AC & VC & AD & WAR \\ \midrule
FF & 0 & 14.5 & 0.0 & 0.0 & 48.0 & / & /\\
Reducto & / & / & / & / & 48.6 & / & /\\
\textit{InFi}-Skip & 36.1 & 18.9 & 33.1 & 56.0 & 66.5 & 75.4 & 11.6\\
Optimal & 64.8 & 34.4 & 71.8 & 91.2 & 77.7 & 86.8 & 31.1\\ \bottomrule
\toprule
Method & SR & NER & HAR & UI & SC & VC-MP & MR \\ \midrule
\textit{InFi}-Skip & 44.1 & 26.8 & 91.2 & 72.4 & 22.5 & 70.7 & 40.9\\
Optimal & 59.9 & 34.4 & 91.8 & 79.8 & 63.8 & 77.7 & 59.9 \\ \bottomrule
\end{tabular}
\end{table}

\begin{table}[t]
\centering
\caption{Filtering rate @ 90\% inference accuracy of REUSE methods.}
\label{tab:reuse-90acc}
\begin{tabular}{@{}lllllll@{}}
\toprule
Method & GC & AC & HAR & SC & VC-MP & VC \\ \midrule
FC & 66.1\% & 13.2\% & / & / & / & 59.4\% \\
\textit{InFi}-Reuse & \textbf{98.8\%} & \textbf{32.1\%} & \textbf{98.3\%}  & \textbf{43.4\%} & \textbf{95.0\%} & \textbf{91.1\%} \\ \bottomrule
\end{tabular}
\end{table}

\subsection{Inference Accuracy vs. Filtering Rate}
\label{subsec:acc-r}

First, we test two device-independent metrics (inference accuracy and filtering rate) on the ten inference workloads.
We adjust the confidence threshold in FF, Reducto, and \textit{InFi}-Skip, and the ratio of cached inputs in FC and \textit{InFi}-Reuse, from 0 to 1 with 0.01 interval. 

\textbf{Redundancy measurements.}
(1) SKIP:
For FD (PE), outputs with no detected face (person keypoints) are redundant.
For GC (SC), outputs with classification confidence less than a threshold, CONF (0.9), are redundant.
For AC, outputs that are not in a subset of classes, Sub, are redundant.
For SR, outputs with the number of recognized words less than a threshold, N, are redundant.
For NER, outputs without entity label ``PERSON'' are redundant.
For HAR, outputs that are not ``LAYING'' are redundant.
For UI, outputs that do not belong to the first 6 users are redundant.
For AD, outputs that are not in \{``Cry, Sneeze, Firing''\} (anomaly events) are redundant.
For MR, we randomly select half of the radio modulation types as redundant.
For WAR, outputs with ``NO PERSON'' are redundant.
For VC and VC-MP, outputs with zero count are redundant.
(2) REUSE:
Experimental results show that cache miss happens rarely, so the homogeneity threshold is set as 0.5.
We regard inputs that hit the cache as redundant.
For the VC (-MP), since we have 86K images from each camera, a fixed cache ratio can lead to serious inefficiency in the KNN algorithm.
We fix the cache size as 1000 and reinitialize the cache every 5000 frames.
For other inference workloads, we set a fixed cache size according to the cache ratio.

\textbf{Overview of results.}
Tab.~\ref{tab:skip-90acc} and Tab.~\ref{tab:reuse-90acc} summarize the results of the SKIP and REUSE methods.
Following related work~\cite{reducto}, we report the filtering rates at 90\% inference accuracy.
The optimal results are computed by (1-0.9)+$r_N$ where $r_N$ denotes the ratio of redundant inputs in the test dataset.
Results show that \textit{InFi}-Skip outperforms FF and Reducto on all 10 workloads with significantly higher filtering rates and wider applicability. Similarly,  \textit{InFi}-Reuse significantly outperforms FC on all 6 applicable workloads.
\textit{InFi}-Skip can filters 18.9\%-91.2\% inputs and \textit{InFi}-Reuse can filters 32.1\%-98.8\% inputs, while keeping more than 90\% inference accuracy.
For all workloads, \textit{Low-level} method cannot achieve 90\% inference accuracy unless no input is filtered (i.e. 0.0\% filtering rate), and we omit these results in the tables.

\begin{figure}[t]
     \centering
     \begin{subfigure}[b]{0.49\linewidth}
         \centering
         \includegraphics[width=\linewidth]{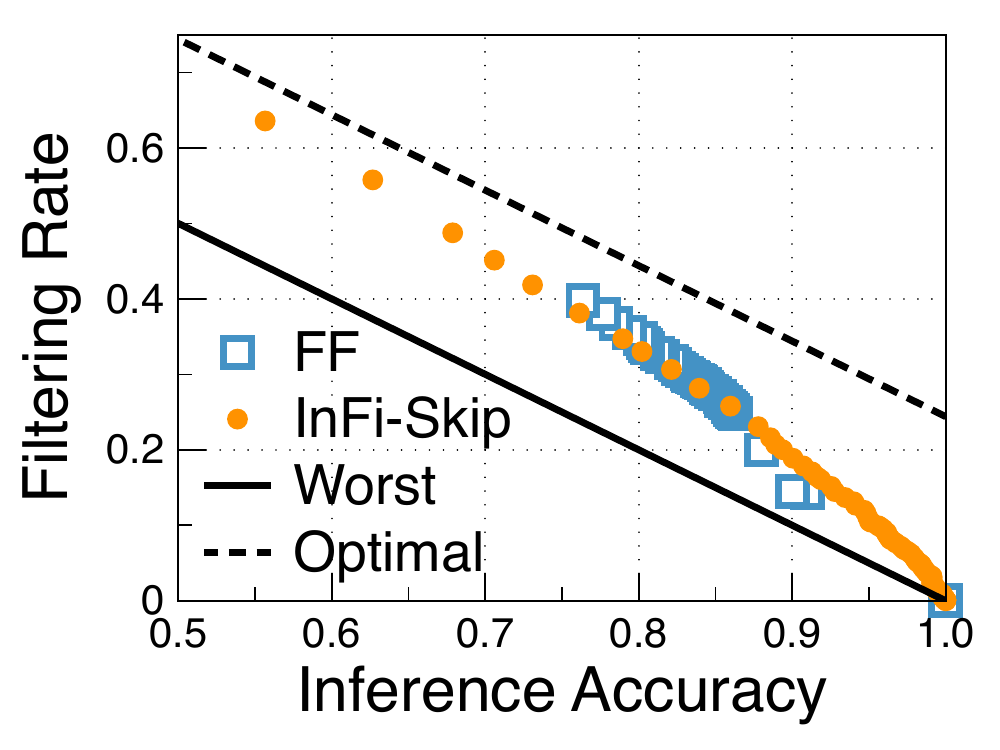}
         \caption{Pose Estimation}
         \label{fig:pd}
     \end{subfigure}
     \hfill
     \begin{subfigure}[b]{0.49\linewidth}
         \centering
         \includegraphics[width=\linewidth]{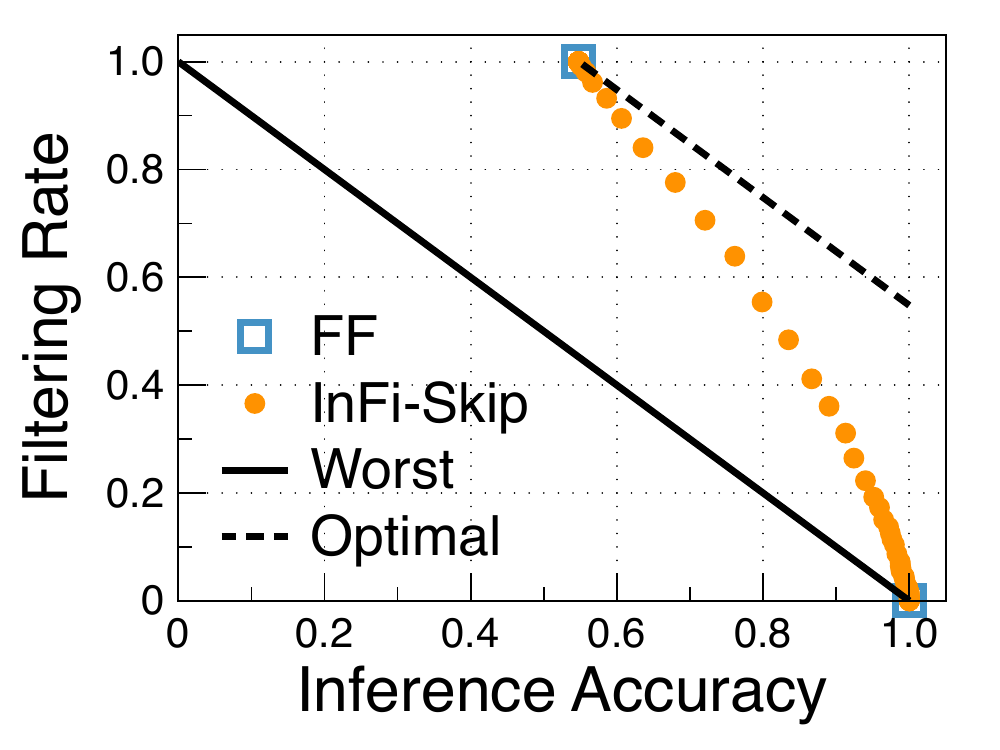}
         \caption{Face Detection}
         \label{fig:fd}
     \end{subfigure}
\caption{Comparison between FF and \textit{InFi}-Skip filters on visual detection workloads.}
\label{fig:detection}
\end{figure}

\begin{figure}[t]
     \centering
     \begin{subfigure}[b]{0.49\linewidth}
         \centering
         \includegraphics[width=\linewidth]{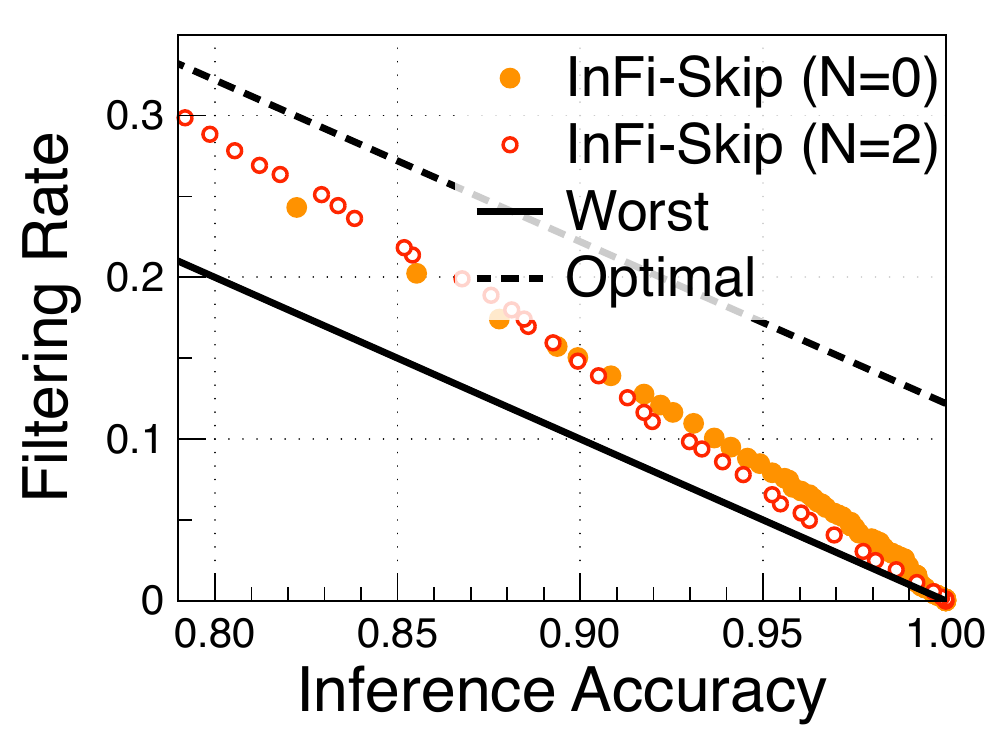}
         \caption{SR (N=0)}
         \label{fig:sr-n0}
     \end{subfigure}
     \hfill
     \begin{subfigure}[b]{0.49\linewidth}
         \centering
         \includegraphics[width=\linewidth]{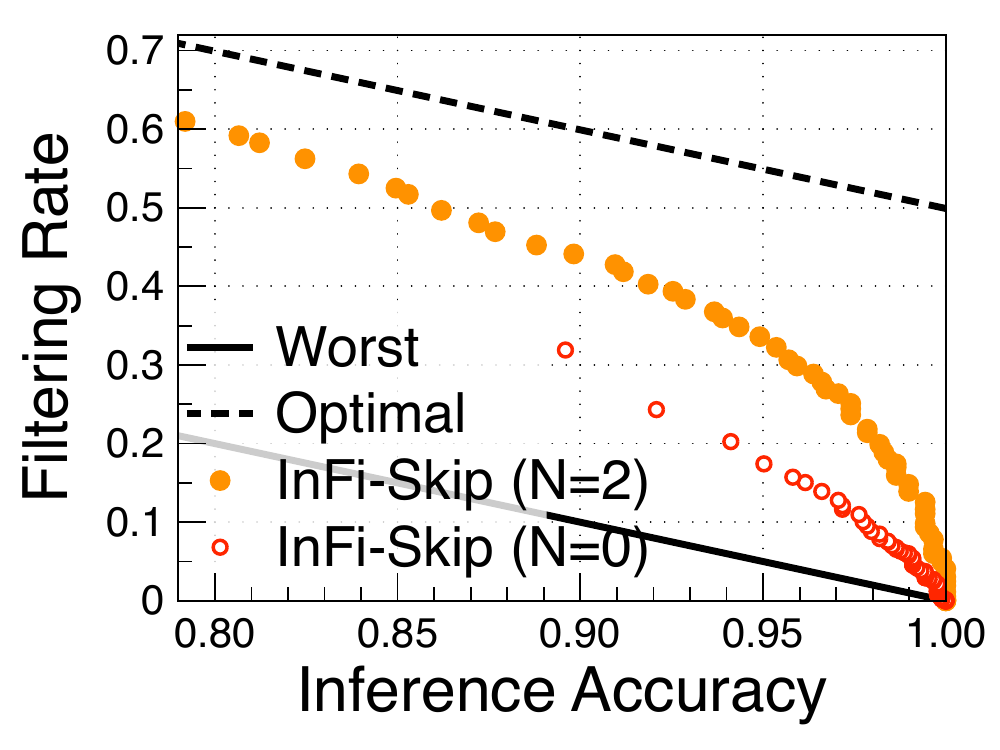}
         \caption{SR (N=2)}
         \label{fig:sr-n2}
     \end{subfigure}
\caption{\textit{InFi}-Skip filters on speech recognition workloads. N is the minimal number of recognized words.}
\label{fig:skip-sr}
\end{figure}

\begin{figure}[t]
     \centering
     \begin{subfigure}[b]{0.49\linewidth}
         \centering
         \includegraphics[width=\linewidth]{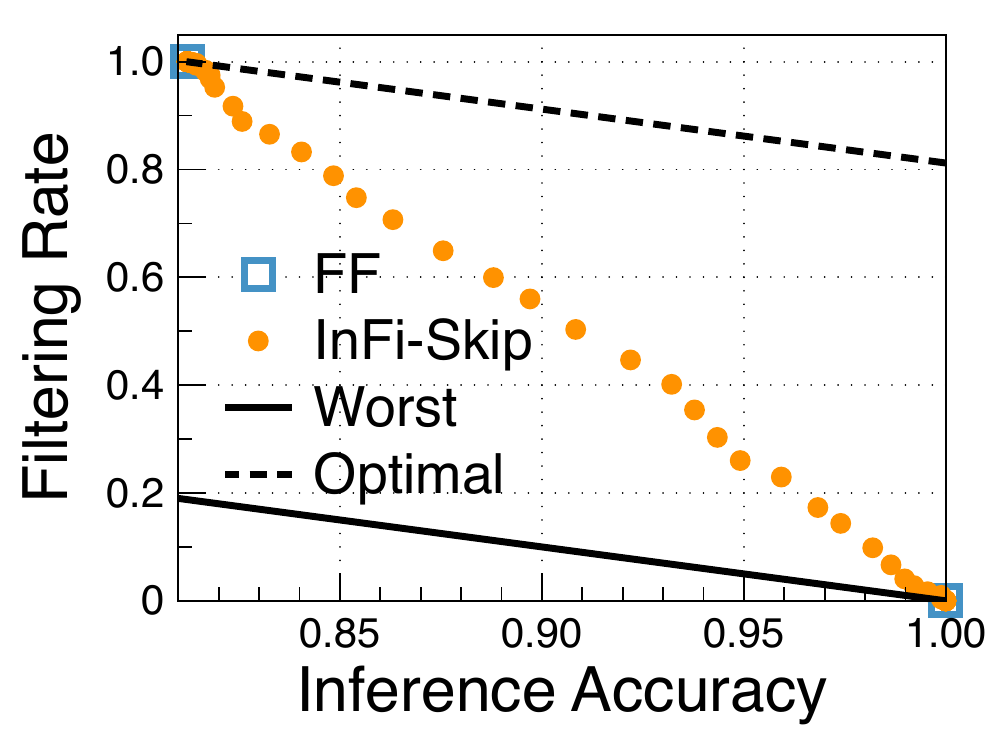}
         \caption{AC (2 Classes)}
         \label{fig:ac-2c}
     \end{subfigure}
     \hfill
     \begin{subfigure}[b]{0.49\linewidth}
         \centering
         \includegraphics[width=\linewidth]{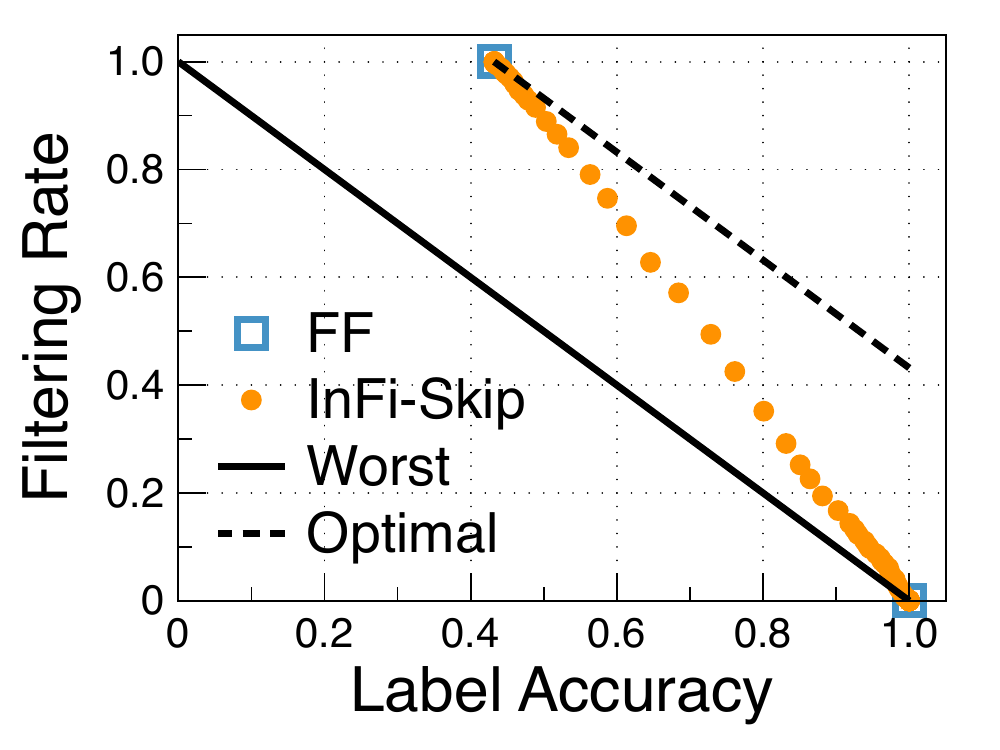}
         \caption{AC (8 Classes)}
         \label{fig:ac-8c}
     \end{subfigure}
\caption{\textit{InFi}-Skip filters on action classification workloads with 2/8 selected classes in the subset.}
\label{fig:skip-ac}
\end{figure}

\textbf{Feature discriminability.}
By comparing FF and \textit{InFi} on FD, PE, GC, and AC workloads, we evaluate the discriminability of our end-to-end learned features.
As shown in Fig.~\ref{fig:detection}, FF works on the pose estimation workload, but not on the face detection workload.
The ``Worst'' case is calculated by $r=1-Acc$.
The reason may be that there is a ``person'' label in the ImageNet dataset, so the pre-trained feature embedding in FF is discriminative for determining whether there is a human pose.
However, on other tasks (e.g., FD, GC, and AC), the pre-trained features are not discriminative and FF can only provide two extreme filtering policies: either filtering all input or filtering nothing, which is useless in practice.
On the contrary, \textit{InFi}-Skip learns feature embedding with robust discriminability and performs well on all four workloads.
With over 90\% inference accuracy, \textit{InFi}-Skip can filter 18.9\% and 36.1\% inputs for PE and FD workloads, respectively.

\textbf{Transferability.}
One interesting question is, how transferable is the trained filter to workloads with a looser or tighter redundancy measurement?
We set the minimal number of recognized words, N, as 0 and 2 and train two \textit{InFi}-Skip filters.
Then we test the two filters on two test sets with different N.
As shown in Fig.~\ref{fig:skip-sr}, the performance of \textit{InFi}-Skip (N=2) is close to \textit{InFi}-Skip (N=0) when tested with N=0, however, the performance of \textit{InFi}-Skip (N=0) is apparently worse when tested with N=2.
An intuitive explanation is that the learned feature with a looser redundancy measurement covers the one with a tighter redundancy measurement, while the opposite is not true.

\textbf{Sensitivity to class subset size.}
For the class-subset redundancy measurement, we set different subset sizes to test the sensitivity.
As shown in Fig.~\ref{fig:skip-ac}, for action classification workload, setting a smaller class subset brings more redundant samples.
And \textit{InFi}-Skip robustly provides smooth accuracy-efficiency trade-off curves in both cases, which significantly outperforms FF (only two extreme points are provided).

\begin{figure}[t]
     \centering
     \begin{subfigure}[b]{0.49\linewidth}
         \centering
         \includegraphics[width=\linewidth]{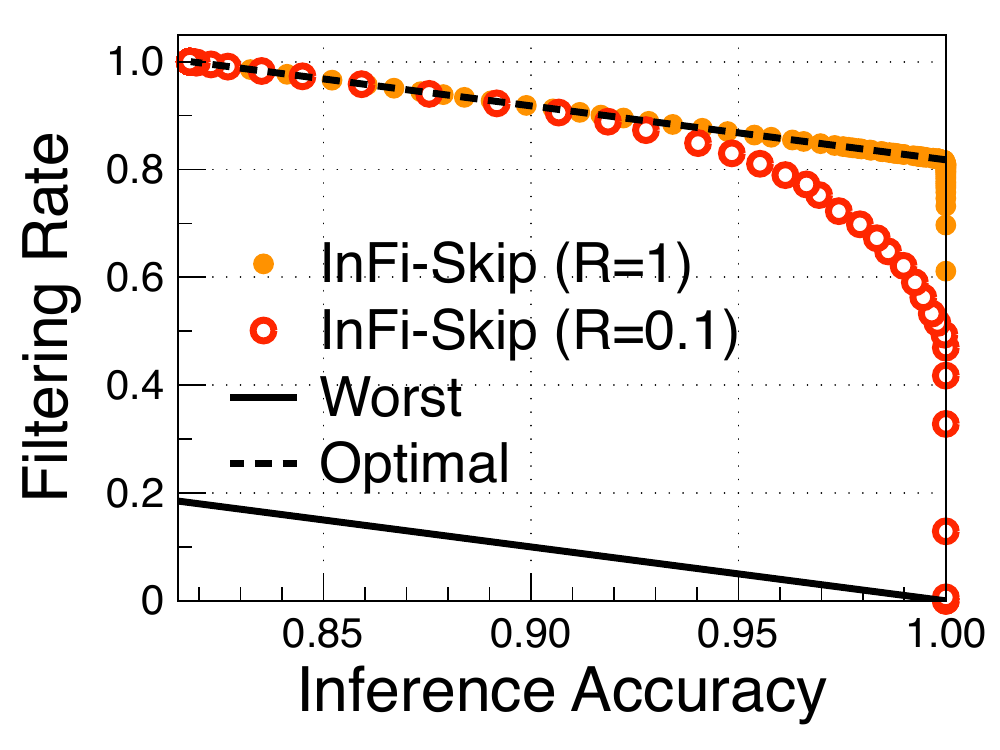}
         \caption{HAR (L=LAYING)}
         \label{fig:har-skip}
     \end{subfigure}
     \hfill
     \begin{subfigure}[b]{0.49\linewidth}
         \centering
         \includegraphics[width=\linewidth]{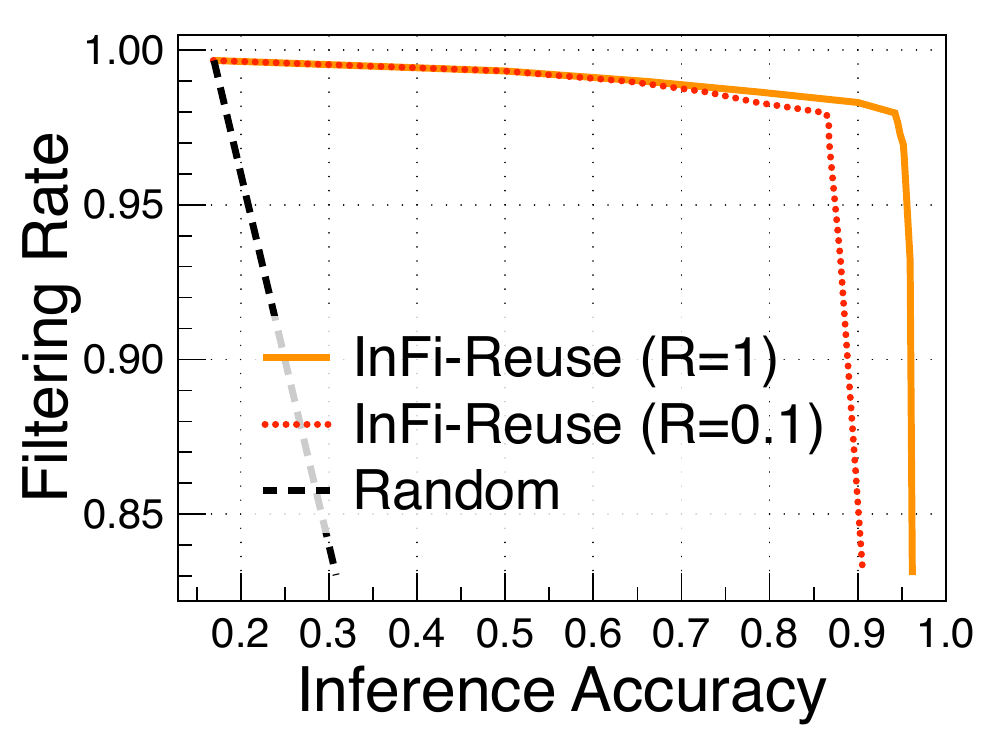}
         \caption{HAR Reuse}
         \label{fig:har-reuse}
     \end{subfigure}
\caption{\textit{InFi} filters on HAR inference workloads. R denotes the ratio of training samples used. The ``Random'' case labels each input randomly.}
\label{fig:har}
\end{figure}

\begin{figure}[t]
     \centering
     \begin{subfigure}[b]{0.49\linewidth}
         \centering
         \includegraphics[width=\linewidth]{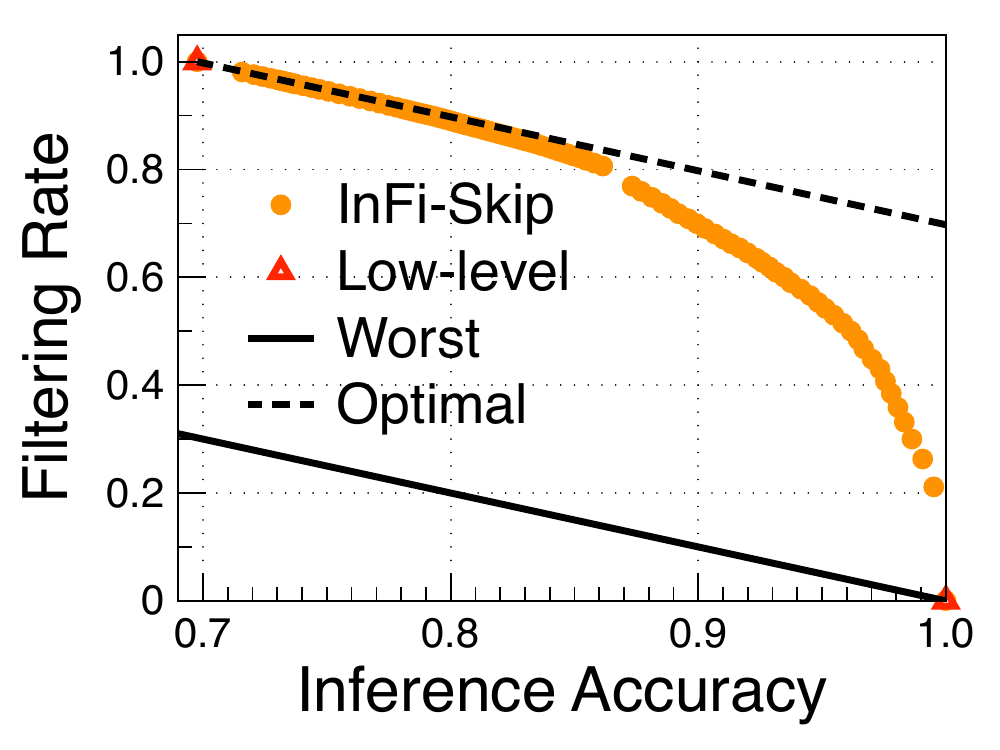}
         \caption{NDense=200, EmbLen=128}
         \label{fig:ui-skip}
     \end{subfigure}
     \hfill
     \begin{subfigure}[b]{0.49\linewidth}
         \centering
         \includegraphics[width=\linewidth]{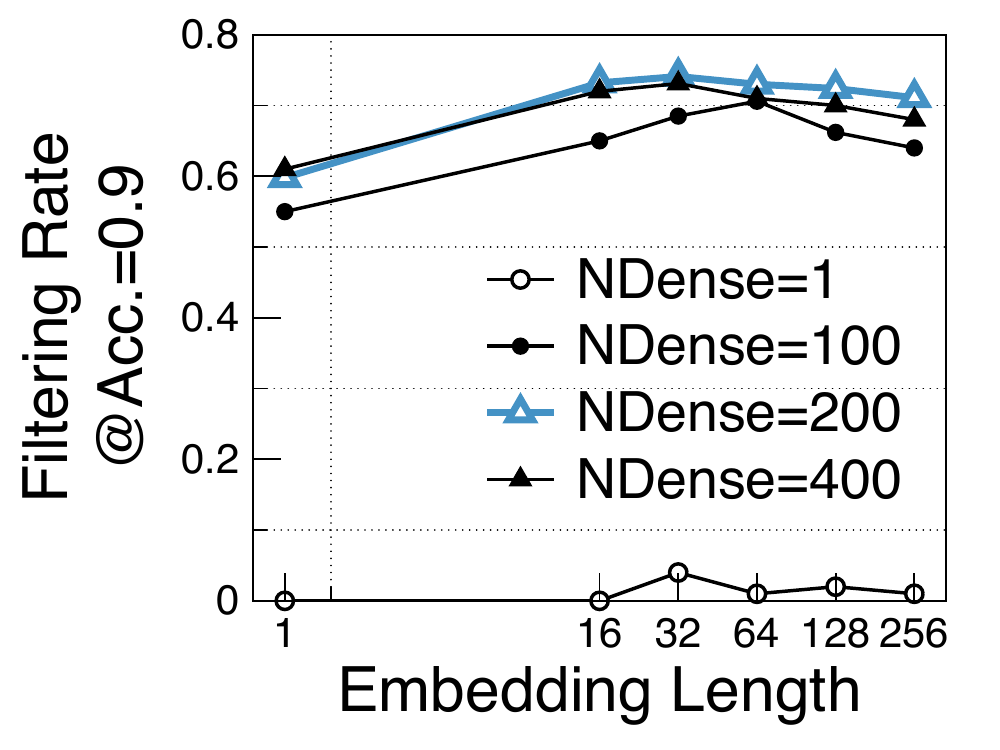}
         \caption{Model Complexity}
         \label{fig:ui-complex}
     \end{subfigure}
\caption{\textit{InFi}-Skip on UI workload. NDense is the number of dense units. EmbLen is the length of embedding.}
\label{fig:ui}
\end{figure}

\textbf{Sensitivity to training size.}
We further divide training splits into sets with different sizes.
As shown in Fig.~\ref{fig:har}, using only 10\% samples from the training set, \textit{InFi} can still achieve near-optimal performance on HAR workload.
Let R denote the ratio of training samples used for training.
When achieving over 95\% inference accuracy, \textit{InFi}-Skip (R=1) filters 86.4\% inputs while \textit{InFi}-Skip (R=0.1) still filters 81.1\%.
For high-accuracy reuse, the impact of training size is relatively greater.
When filtering 90\% inputs, \textit{InFi}-Reuse (R=1) can achieve 95.9\% inference accuracy, while the accuracy of \textit{InFi}-Reuse (R=0.1) decreases to 88.1\%.

\textbf{Sensitivity to model complexity.}
To explore the relationship between the complexity and performance of input filters, we trained \textit{InFi}-Skip filters for the UI workload using the different lengths of embedding (1, 16, 32, 64, 128, 256) and the number of dense units (1, 100, 200, 400) in the classifier.
And we measure the performance by the maximum filtering rate when achieving 90\% inference accuracy.
As shown in Fig.~\ref{fig:ui-complex}, except for extreme cases (e.g., single dense or embedding unit), the filtering performance is relatively robust.

\textbf{Sensitivity to K in KNN.}
The parameter K in KNN affects the classification accuracy.
We vary K from 1 to 20 and test the REUSE filters' performance.
As shown in Fig.~\ref{fig:classify}, on GC workload, \textit{InFi}-Reuse is robust to varied K parameters, while FC suffers serious performance degradation.
For example, with 90\% inference accuracy, FC (K=5) can filter 68.4\% inputs, while FC (K=1) can only filter 27.3\% which is slightly higher than the random guess (20\%).
On the contrary, \textit{InFi}-Reuse (K=1,5) can all achieve a 94.3\% filtering rate with more than 95\% inference accuracy.
For the AC workload, the results show that the handcrafted feature SIFT is not discriminative, and all tested K parameters lead to similar performance with random labeling.
\textit{InFi}-Reuse can learn an action-related discriminative feature, it can filter 18.6\% inputs and keep more than 90\% inference accuracy (K=10).

\begin{figure}
     \centering
     \begin{subfigure}[b]{0.49\linewidth}
         \centering
         \includegraphics[width=\linewidth]{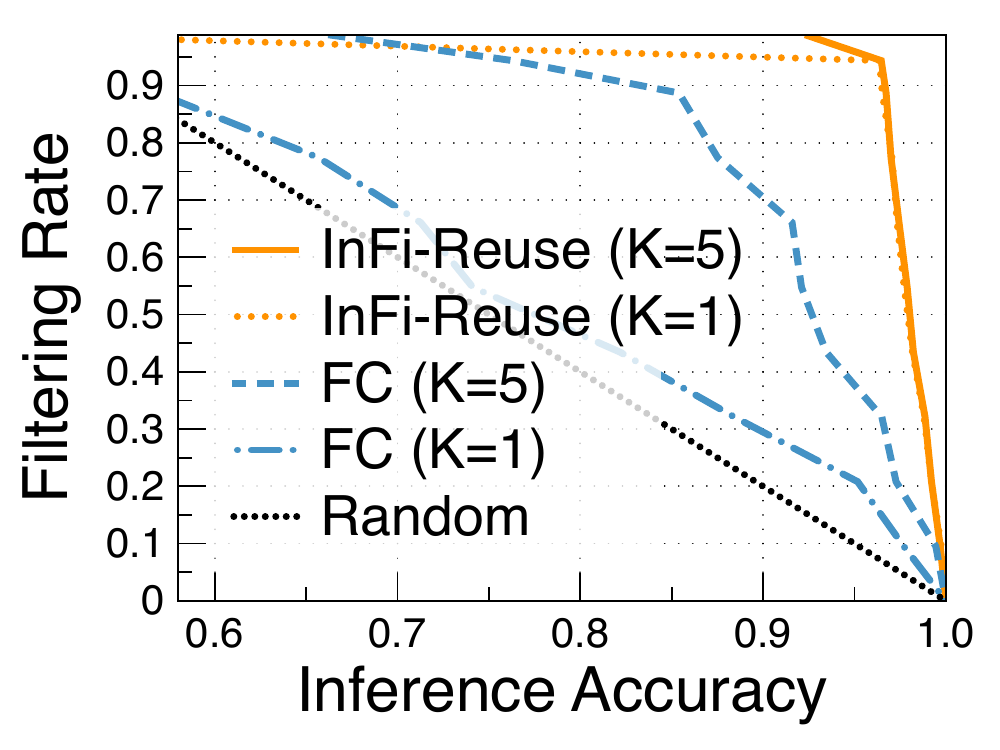}
         \caption{Gender Classification}
         \label{fig:gc-reuse}
     \end{subfigure}
     \hfill
     \begin{subfigure}[b]{0.49\linewidth}
         \centering
         \includegraphics[width=\linewidth]{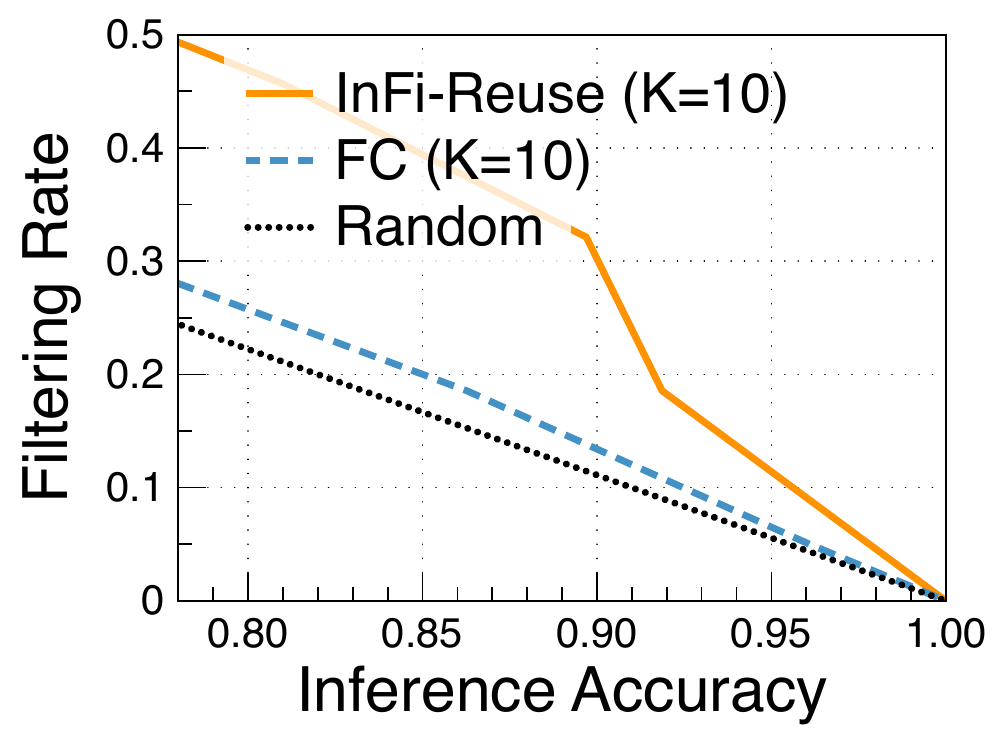}
         \caption{Action Classification}
         \label{fig:ac-reuse}
     \end{subfigure}
\caption{Comparison of FC and \textit{InFi}-Reuse on visual classification workloads. K is the parameter in KNN.}
\label{fig:classify}
\end{figure}

\begin{figure}[t]
     \centering
     \begin{subfigure}[b]{0.49\linewidth}
         \centering
         \includegraphics[width=\linewidth]{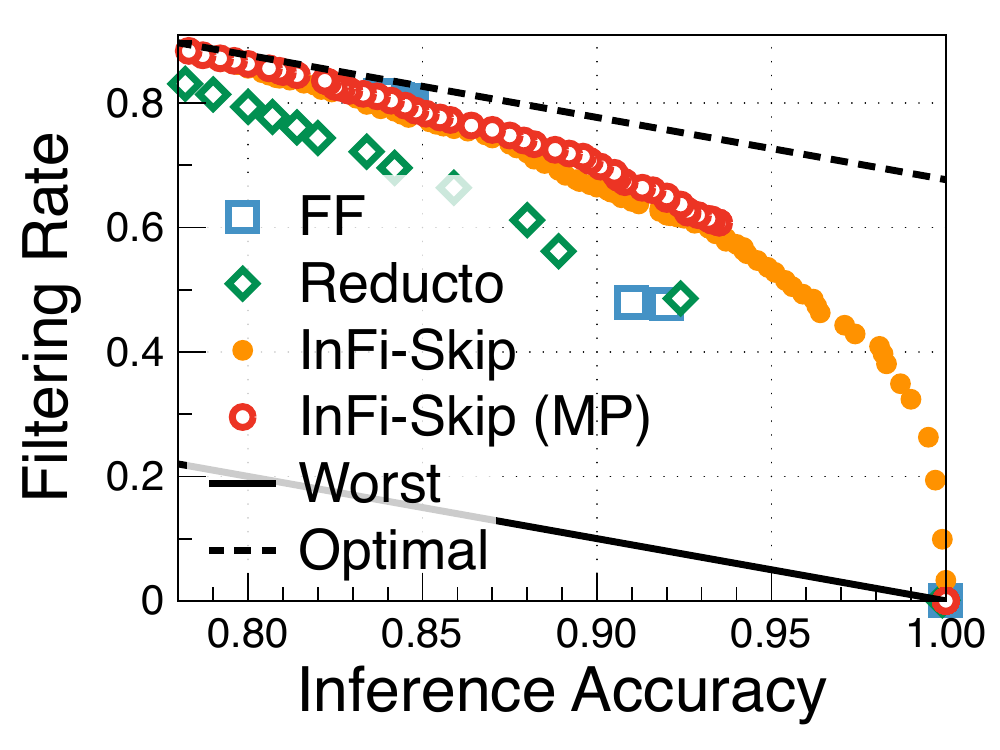}
         \caption{SKIP Methods}
         \label{fig:vc-skip}
     \end{subfigure}
    \hfill
    \begin{subfigure}[b]{0.49\linewidth}
         \centering
         \includegraphics[width=\linewidth]{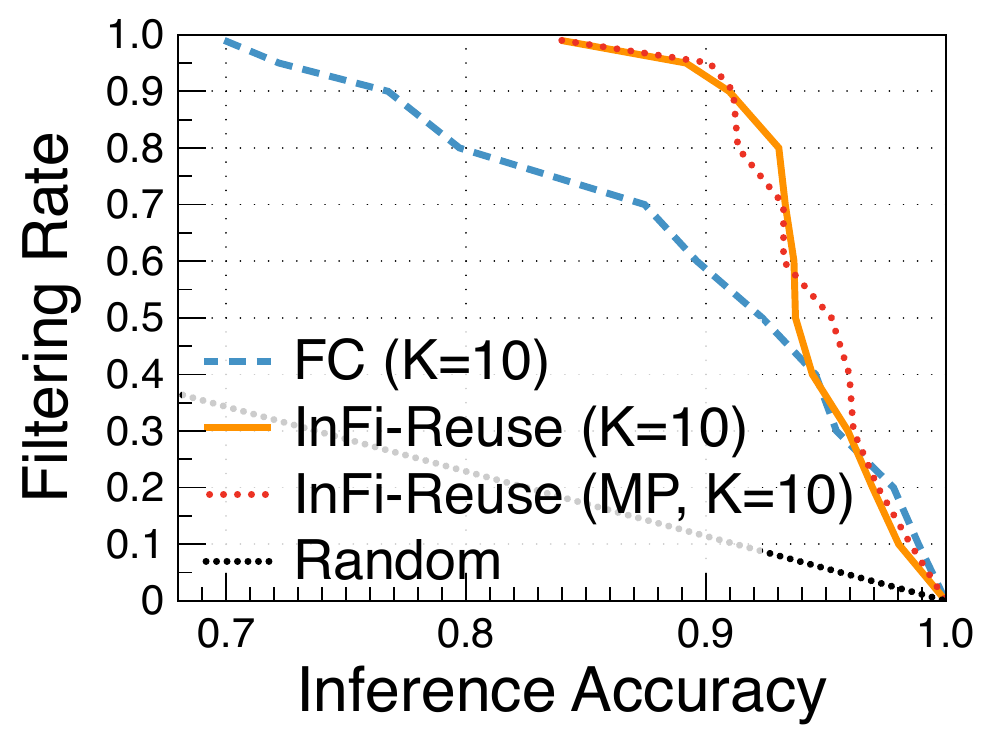}
         \caption{REUSE Methods}
         \label{fig:vc-reuse}
     \end{subfigure}
\caption{Comparisons of filters on VC and VC-MP.}
\label{fig:vc}
\end{figure}

\begin{figure}[t]
     \centering
     \begin{subfigure}[b]{0.49\linewidth}
         \centering
         \includegraphics[width=\linewidth]{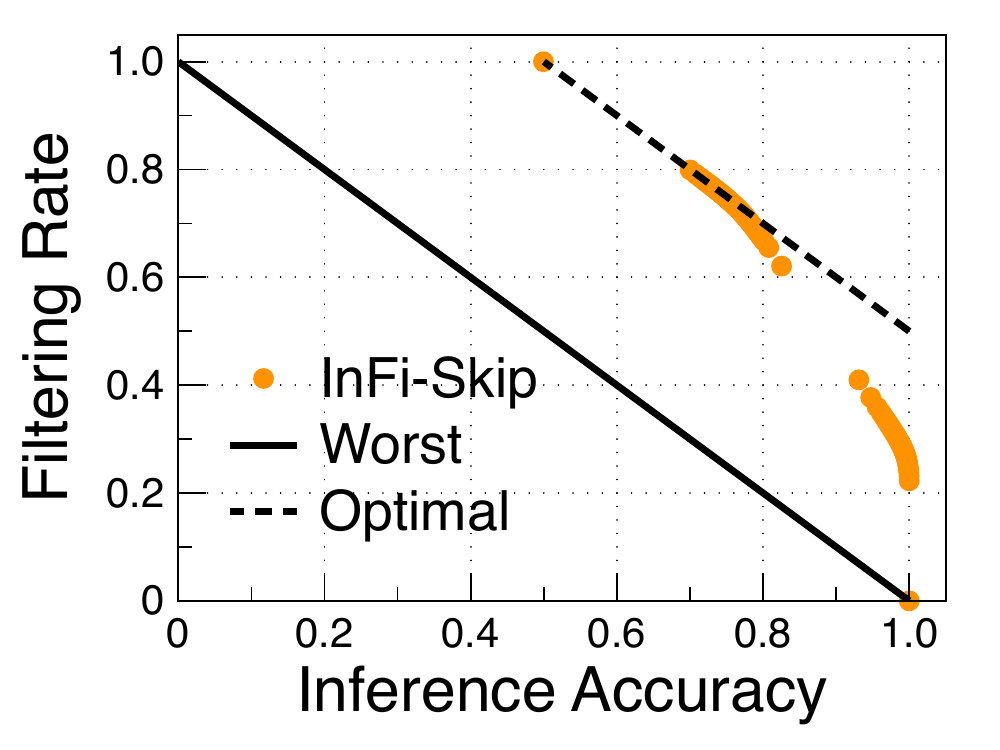}
         \caption{Modulation Recognition}
         \label{fig:mr}
     \end{subfigure}
    \hfill
    \begin{subfigure}[b]{0.49\linewidth}
         \centering
         \includegraphics[width=\linewidth]{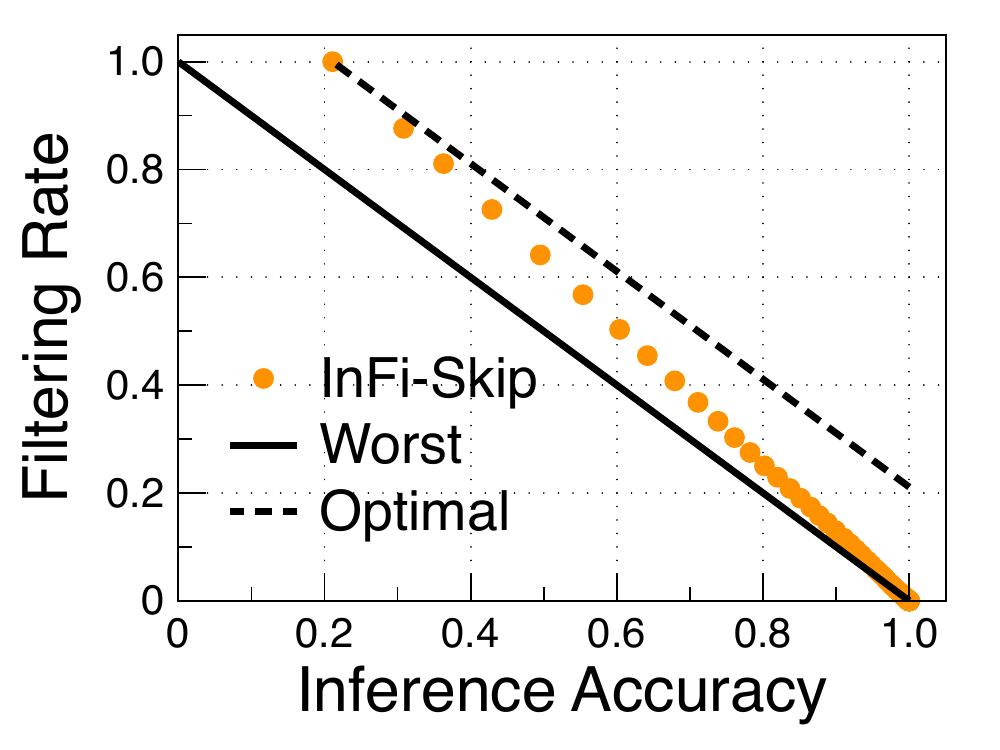}
         \caption{WiFi Action Recognition}
         \label{fig:wifi-har}
     \end{subfigure}
\caption{\textit{InFi}-Skip on on MR and WAR workloads.}
\label{fig:mr-war}
\end{figure}

\begin{figure}[t]
     \centering
     \begin{subfigure}[b]{0.49\linewidth}
         \centering
         \includegraphics[width=\linewidth]{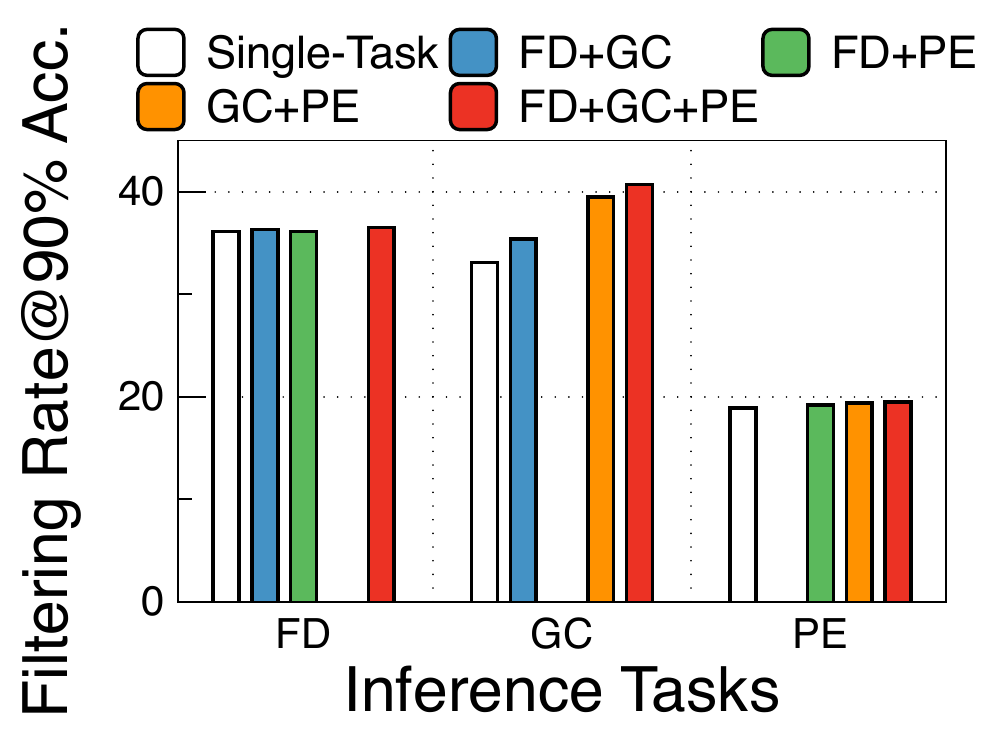}
         \caption{Single-Modality Multi-Task}
         \label{fig:smmt}
     \end{subfigure}
     \hfill
    \begin{subfigure}[b]{0.49\linewidth}
         \centering
         \includegraphics[width=\linewidth]{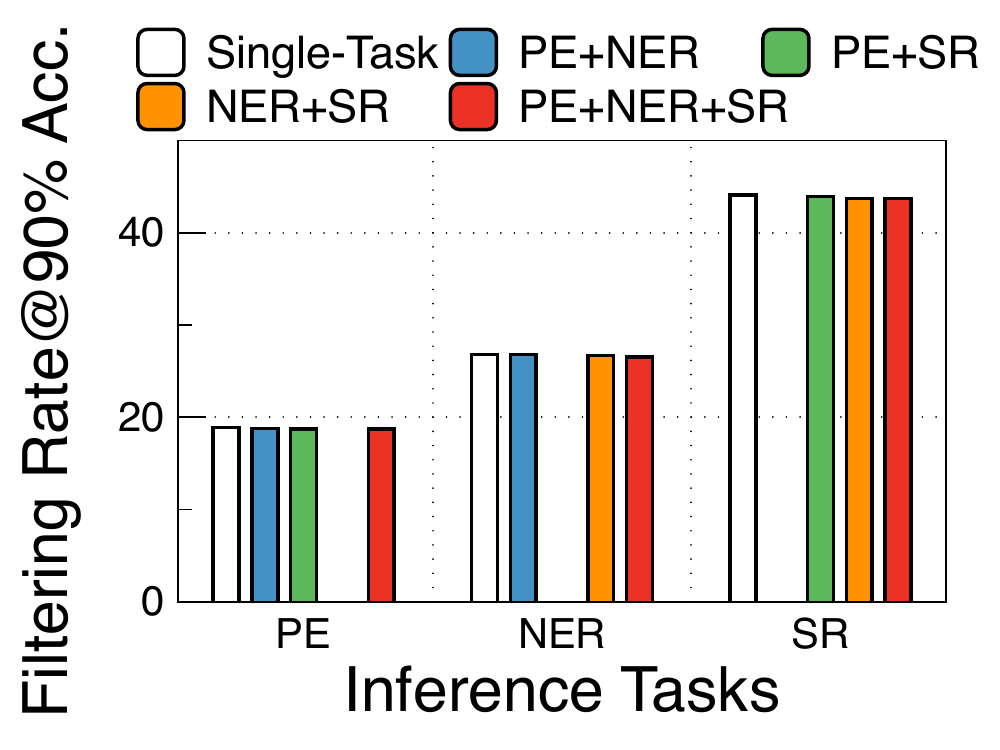}
         \caption{Multi-Modality Multi-Task}
         \label{fig:mmmt}
     \end{subfigure}
\caption{Comparison of single-task and multi-task \textit{InFi}-Skip. The plus sign denotes joint training using multiple tasks.}
\label{fig:mtl-infi}
\end{figure}

\begin{figure}[t]
    \centering
    \includegraphics[width=0.95\linewidth]{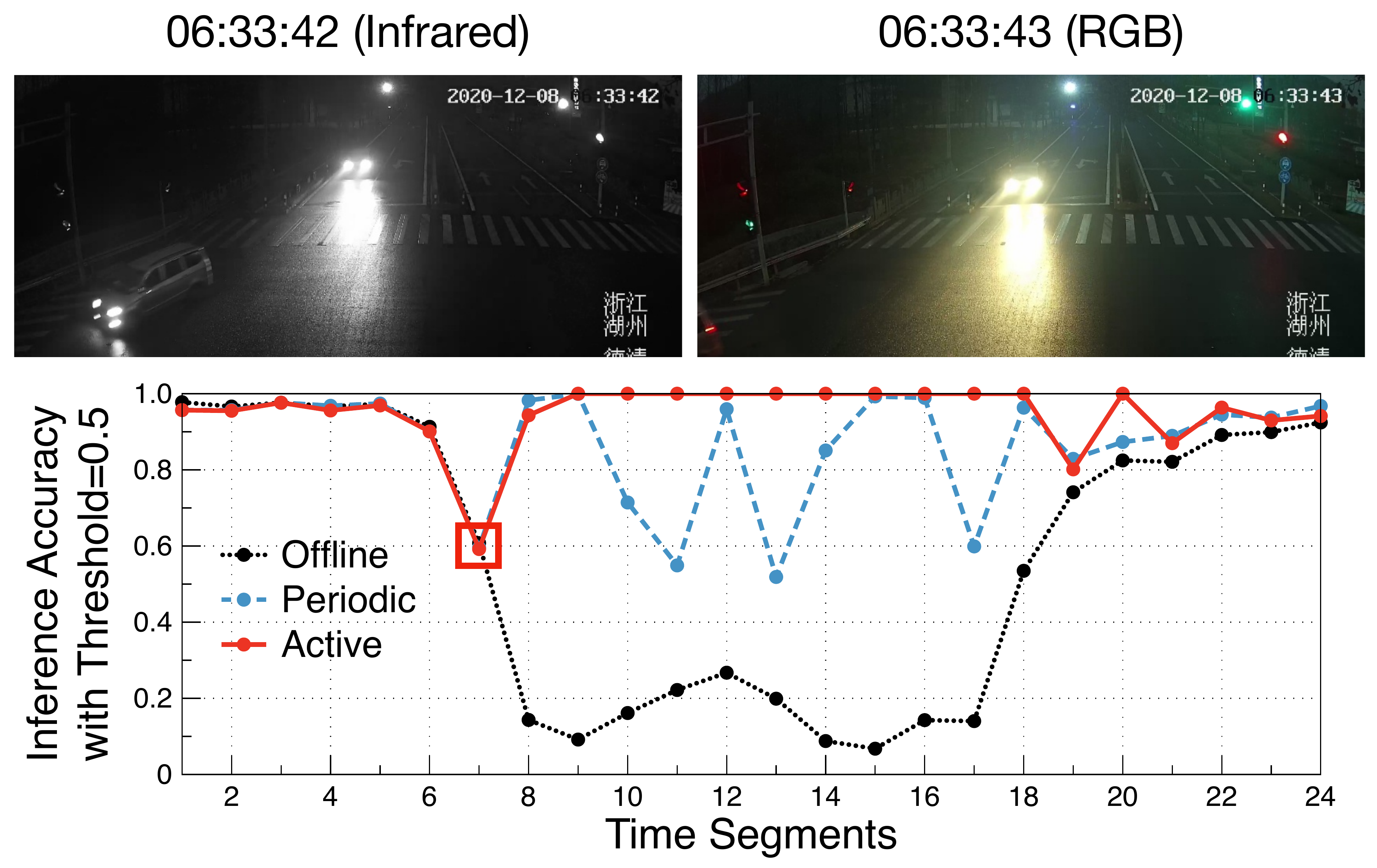}
    \caption{Active update of \textit{InFi}-Skip on VC workload. The two successive frames, one infrared and one RGB, corresponds to the boxed time segment.}
    \label{fig:online-active}
\end{figure}

\textbf{Comparisons on VC(-MP) workloads.}
Unlike other datasets, the video frames arrive in time order rather than randomly.
For VC-MP, we partition the YOLOv3 model to mobile-side (the first 39 layers) and edge-side (the rest layers).
As shown in Fig.~\ref{fig:vc}, \textit{InFi} outperforms FF, Reducto, and FC, and also is the only applicable method for the VC-MP workload.
With over 90\% inference accuracy, \textit{InFi}-Skip achieves 66.5\% filtering rate, while FF and Reducto achieve 48.0\% and 48.6\%, respectively;
\textit{InFi}-Reuse filters 31.7\% more inputs than FC when K=10.
The results show the superiority of end-to-end learned features over handcrafted and pre-trained ones.

\textbf{Mobile-featured modalities.}
Radio signals and WiFi CSI (channel side information) are featured modalities for mobile AI applications.
Fig.~\ref{fig:mr-war} shows the experimental results of applying \textit{InFi} on two mobile-featured tasks: radio modulation recognition and WiFi CSI-based action recognition.
Note that, existing methods cannot be applied to both tasks.
With 90\% target accuracy, \textit{InFi}-Skip saves 40.9\% and 11.6\% computations for MR and WAR workloads, respectively.

\textbf{Multi-task extension.}
In Sec.~\ref{subsec:multitask}, we present how to extend \textit{InFi} to multi-task workloads.
We use the Hollywood2 dataset and corresponding inference tasks to evaluate the multi-task extension of \textit{InFi}-Skip.
First, we select three inference tasks on image modality: FD, GC, and PE.
We build \textit{InFi}-Skip filters using a single task, two tasks, and three tasks and evaluate their performance on each task.
As shown in Fig.~\ref{fig:smmt}, the multi-task filter outperforms single-task ones on all three tasks, improving the filtering rate up to 7.6\% (for the GC task) when achieving 90\% inference accuracy.
Next, we select three inference tasks on different modalities: PE on images, NER on texts, and SR on audio.
And we build \textit{InFi}-Skip filters using a single modality, two modalities, and three modalities and evaluate them.
As shown in Fig.~\ref{fig:mmmt}, fusing these multi-modality tasks into one filter results in a slight decrease in the filtering rate.
But note that the overall efficiency is improved due to the shared parameters among different tasks.

\textbf{Online active update.}
To evaluate the active strategy for online adaptation of \textit{InFi}, we select the VC workload and compare three training methods:
(1) \textit{Offline}: selects the first 10\% frames of a day to train;
(2) \textit{Periodic}: selects the first 10\% frames of each hour to train and update;
(3) \textit{Active}: see Sec.~\ref{subsec:training}.
For a fair comparison, we set the same threshold (0.5) for the three methods.
As shown in Fig.~\ref{fig:online-active}, our proposed active strategy significantly improves the online adaptability of \textit{InFi}-Skip.
The offline policy's performance seriously degrades when the input distribution changes, mainly because the frames change from infrared to RGB images.
On average, the offline policy achieves only 56.4\% inference accuracy.
The periodic update policy alleviates this problem to some extent, improving the average accuracy to 87.0\%, but still suffers from performance fluctuations.
The active strategy's performance only drops at the 7th time segment, as it does not see any RGB images before that.
And the active strategy effectively selects informative samples to fit the new distribution and performs accurate filtering robustly.
On average, our active policy achieves 94.8\% inference accuracy, which is 38.4\% higher than the offline policy.

\subsection{Filterability}
\label{subsec:exp-filterable}

In Sec.~\ref{sec:filterability}, we compare the hypothesis complexity of the inference and filter models.
Let ``Conf.$>$T'' denote the low-confidence classification case ($\S$~\ref{subsec:conf}), ``Class Subset'' denote the redundant class subset case ($\S$~\ref{subsec:subset}), and ``Reg.$>$T'' denote the bounded regression case ($\S$~\ref{subsec:regression}).
GC and SC belong to the ``Conf.$>$T'' case, where T is 0.9.
AC, NER, and HAR belong to the ``Class Subset'' case, where AC selects 2 action labels, NER selects the ``PERSON'' label, and HAR selects the ``LAYING'' label.
FD, PE, and VC(MP) belong to the `Reg.$>$T'' case, where T is 0.
SR is a sequence-to-sequence model, which cannot perfectly fit any of these three cases.
We compute the ratio of the resulting filtering rate to the optimal filtering rate at 90\% inference accuracy to compare the filterability of different cases.
From a practical perspective, we evaluate the overall throughput with and without \textit{InFi}-Skip filters.
As shown in Fig.~\ref{fig:filterable}, the ``Conf.$>$T'' case in which we proved that the filter's complexity is not less than the inference model achieves obviously lower filtering ratio (0.41 median), while other cases in which we proved that the filter tends to be less complex achieve apparently higher ratios (0.71/0.78 medians).
On the other hand, the overall throughput improved by \textit{InFi}-Skip filters on filterable cases is more significant than the non-filterable cases.
In the non-filterable cases, GC and SC, \textit{InFi} achieves around 1.3$\times$ throughput, while in the filterable cases, it can improve the throughput up to 5.92$\times$ and achieves 1.8 and 2.25 medians for regression and subset-class cases, respectively.
These results show the guiding significance of our proven filterability in real applications.

\begin{figure}
     \centering
     \begin{subfigure}[b]{0.98\linewidth}
         \centering
         \includegraphics[width=\linewidth]{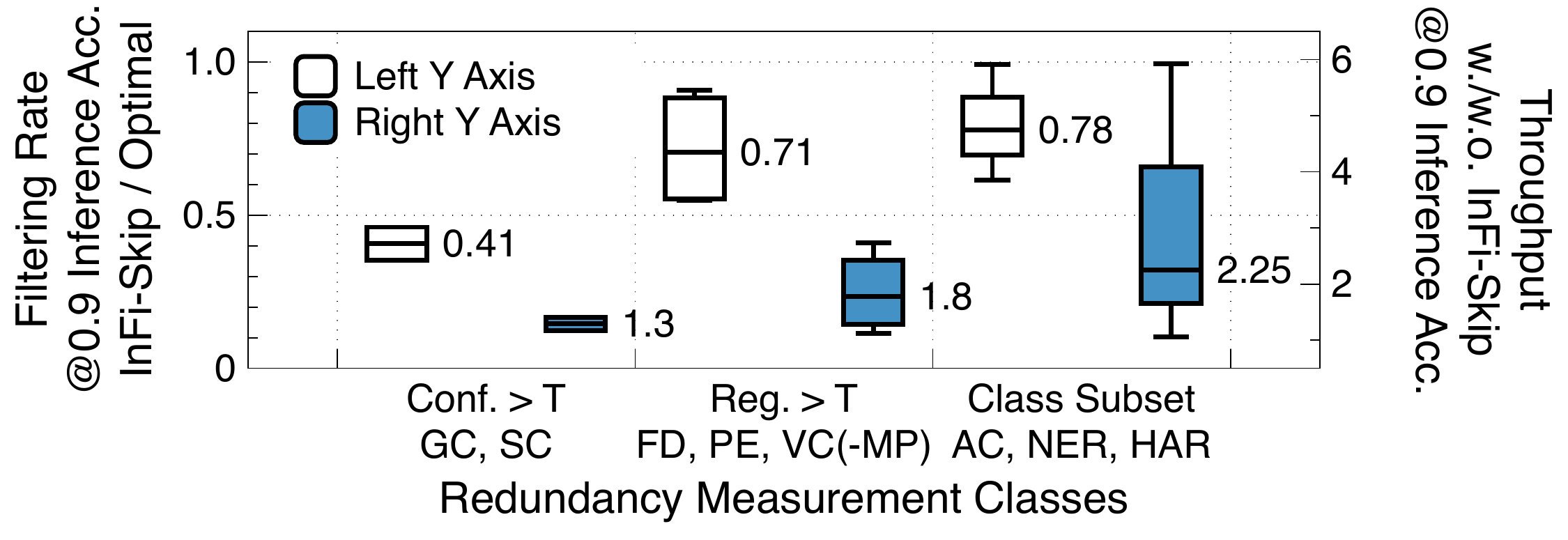}
     \end{subfigure}
\caption{Comparison of filterable and non-filterable cases.}
\label{fig:filterable}
\end{figure}

\begin{figure}[t]
     \centering
     \begin{subfigure}[b]{0.98\linewidth}
         \centering
         \includegraphics[width=\linewidth]{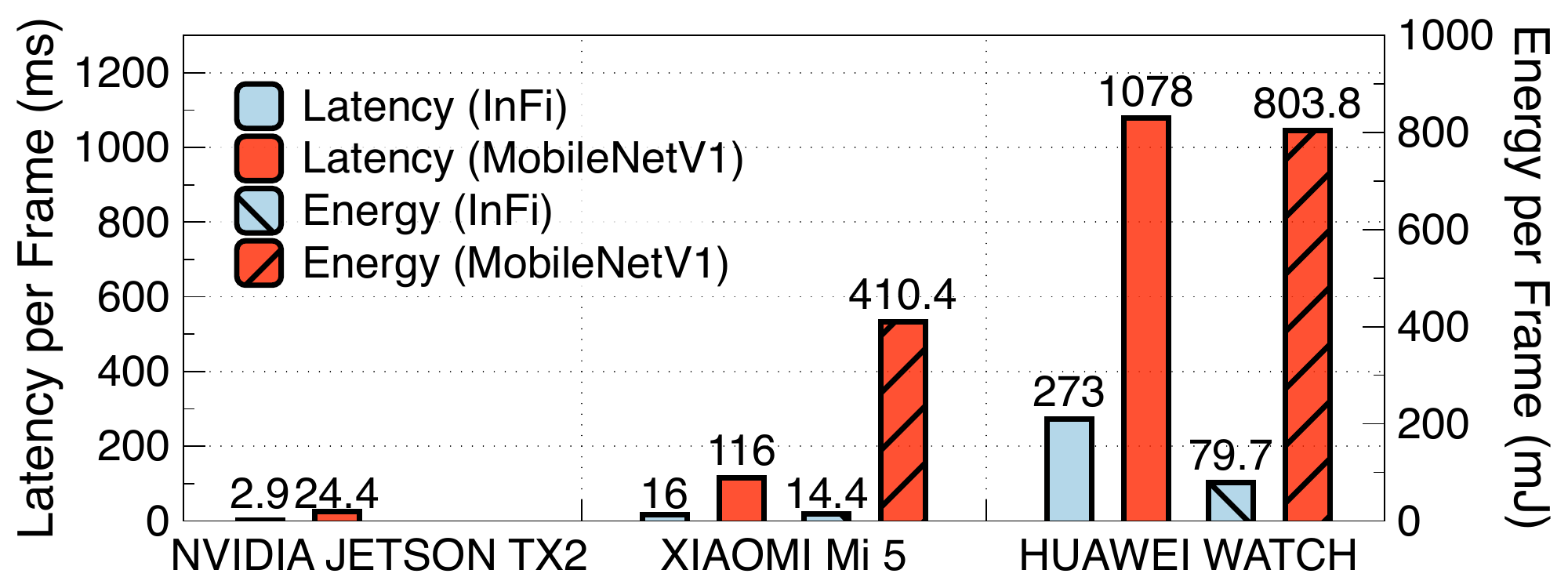}
     \end{subfigure}
\caption{Latency and energy costs of \textit{InFi} (image modality) and MobileNetV1 on mobile platforms.}
\label{fig:latency-energy}
\end{figure}

\subsection{Computation and Resource Efficiency}

As we discussed in Sec.~\ref{subsec:validity}, a ``valid'' filter should be both accurate and lightweight.
The above results have shown that \textit{InFi} can filter a significant amount of inputs while keeping accurate inference.
In the training phase, \textit{InFi} (image modality) takes around 710 ms per batch (batch size is 32) and requires 5337 MB GPU memory which most commercial GPUs can meet.
\textit{InFi} for other input modalities requires far fewer resources, e.g., \textit{InFi} (vector) tasks 3 ms per batch and needs only 435 MB memory.
We test the latency and energy in the inference phase on mobile platforms.
As a fair comparison, we chose the TFLite-optimized MobileNetV1, which is one of the most efficient CNNs on mobile devices.
As shown in Fig.~\ref{fig:latency-energy}, on three mobile platforms, \textit{InFi} with the image feature network costs only 12-25\% runtime of MobileNetV1.
The average energy costs of \textit{InFi} are 14.4/79.7 mJ per frame, which are much lower than MobileNetV1 (410.4/803.8 mJ per frame) on the phone/smartwatch.
We implement \textit{InFi} with MindSpore and the results show that \textit{InFi}'s low-energy consumption and low-latency execution do not depend on the implementation framework.

\textbf{On-device online update.}
Based on the Chaquopy library, we tested the overhead of training on a mobile phone (XIAOMI Mi 5) and a smartwatch (HUAWEI WATCH).
We randomly generated images with the shape (224, 224, 3) and set the batch size as 16.
For \textit{InFi} (image modality), experiments show that it takes around 20 s and 50 s per batch to online update weights on the phone and the watch, respectively.
And for NVIDIA JETSON TX2, training the input filter with the same configurations takes around 1s per batch.

\subsection{Different Mobile-centric Deployments}

Now we evaluate the overall performance of inference workloads in real systems with three ways of deployments.

\textbf{Vehicle counting.}
First, we consider the vehicle counting workload:
1) on-device: \textit{InFi} (image) and YOLOv3 model on TX2;
2) offload: \textit{InFi} (image) on TX2 and YOLOv3 model on edge;
3) model partitioning (MP): first 39 layers (10 convolution blocks) of YOLOv3 and \textit{InFi} (feature map) on TX2, rest of YOLOv3 on the edge server.
The average throughput of the YOLOv3 model on TX2 and edge is 3.2 FPS and 22.0 FPS, respectively.
For MP deployment, the edge-side model serves 24.5 FPS.
We report the average throughput and the bandwidth saving of using \textit{InFi}-Skip and \textit{InFi}-Reuse, with over 90\% inference accuracy, in Tab.~\ref{tab:vc-metrics}.
As a fair comparison, we test the throughput of YOLOv3-tiny~\cite{yolov3-tiny} model, a compressed version for YOLOv3.
The inference accuracy of YOLOv3-tiny is only 67.9\% which does not meet the 90\% target.
Breaking down the overheads, \textit{InFi}'s inference costs around 3 ms per frame, and the average latency of KNN is 6 ms per frame with K=10 and cache size=1000.
Achieving over 90\% inference accuracy, \textit{InFi}-Skip improves the throughput to 9.3/55.2/39.0 FPS for on-device/offload/MP deployments, respectively.
Apparently, in vehicle counting workloads, there are more filtering opportunities for \textit{InFi}-Reuse.
\textit{InFi}-Reuse improves the throughput to 27.2/77.2/46.0 FPS for these three deployments.
Except for the on-device deployment that does not involve cross-device data transmission, \textit{InFi}-Skip / \textit{InFi}-Reuse also save 66.5\% / 91.1\% and 70.7\% / 95.0\% bandwidth for offloading and MP workloads.
Unlike YOLOv3-tiny which trades a significant and fixed loss of accuracy for efficiency, \textit{InFi} provides a flexible trade-off between the inference accuracy and overheads.

\textbf{Pose estimation.}
Second, we evaluate the pose estimation workload:
1) on-device: \textit{InFi} (image) and OpenPose model on TX2;
2) offload: \textit{InFi} (image) on TX2 and OpenPose model on edge;
3) model partitioning (MP): first 39 layers (10 convolution blocks) of OpenPose and \textit{InFi} (feature map) on TX2, rest of OpenPose on the edge server.
Also, we test the throughput of OpenPose-light~\cite{light-openpose} model, a lightweight version of OpenPose.
Experimental results are shown in Tab.~\ref{tab:pe-deploy}.
Similar to the vehicle counting workload, the lightweight model cannot achieve our target 90\% inference accuracy, although its throughput boosts significantly.
\textit{InFi}-Skip can flexibly balance the inference accuracy and throughput.
For example, for the on-device deployment, the throughput improves to 1.17$\times$ after using \textit{InFi}-Skip and the inference accuracy keeps over 90\%.

\textbf{Natural language processing workloads.}
Third, we test two NLP workloads, NER and SC, with different deployments.
Note that, model partitioning is not applicable to our NER workload due to black-box APIs.
As shown in Tab.~\ref{tab:nlp-deploy}, \textit{InFi}-Skip effectively improves throughput and saves 22.5\% and 26.8\% offloading communication for the two tasks, respectively.

\begin{table}[t!]
\caption{Throughput (FPS) / Bandwidth saving (\%) of vehicle counting workloads. Acc. denotes inference accuracy, compared with vehicle count results of YOLOv3.}
\label{tab:vc-metrics}
\centering
\begin{tabular}{@{}l|ccc|c@{}}
\toprule
Workload & YOLOv3 & \textit{InFi}-Skip & \textit{InFi}-Reuse & YOLOv3-tiny \\ \midrule
Acc. (\%) & 100 & 90.3 & 90.5 & 67.9 \\ \midrule
On-device & 3.2/- & 9.3/- & 27.2/- & 20.4/- \\
Offloading & 22.0/- & 55.2/66.5 & 77.2/91.1 & 225.3/- \\
MP & 24.5/- & 39.0/70.7 & 46.0/95.0 & 230.4/- \\ \bottomrule
\end{tabular}
\end{table}

\begin{table}[t!]
\caption{Throughput (FPS) / Bandwidth saving (\%) of pose estimation workloads.}
\label{tab:pe-deploy}
\centering
\begin{tabular}{@{}l|cc|c@{}}
\toprule
Workload & OpenPose & \textit{InFi}-Skip  & OpenPose-light \\ \midrule
Inference Accuracy (\%) & 100 & 90.1 &  76.5 \\ \midrule
On-device & 15.4/- & 18.0/- &  28.1/- \\
Offloading & 27.7/- & 31.5/18.9 &  98.5/- \\
MP & 29.2/- & 33.1/20.2 & 102.4/- \\ \bottomrule
\end{tabular}
\end{table}

\begin{table}[t]
\caption{Throughput (QPS) / Bandwidth saving (\%) of two NLP workloads.}
\label{tab:nlp-deploy}
\centering
\begin{tabular}{@{}l|cc|cc@{}}
\toprule
Workload & NER & NER+\textit{InFi}-Skip & SC & SC+\textit{InFi}-Skip \\ \midrule
Acc. (\%) & 100 & 90.2 & 100 & 90.0 \\ \midrule
On-device & 24.3/- & 33.2/- & 27.9/- & 36.0/- \\
Offloading & 133.2/- & 181.9/26.8 & 60.2/- & 77.7/22.5 \\ 
MP & N/A & N/A & 62.5/- & 82.0/24.1 \\
\bottomrule
\end{tabular}
\end{table}
\section{Related Work}

\textbf{Frame filtering.}
NoScope~\cite{noscope} trains task-specific difference detectors to choose necessary frames for object queries in the video database.
FilterForward~\cite{ff} leverages MobileNet and trains a binary micro-classifier on the intermediate output of a selected layer to determine whether to transmit the input image to the server with the offloaded model.
Reducto~\cite{reducto} performs on-device frame filtering by thresholding the difference of low-level features between successive frames. 
Through elaborate selection for different tasks, low-level features can efficiently and accurately measure the difference.

\textbf{Inference caching.}
Potluck~\cite{potluck} stores and shares inference results between augmented reality applications.
It dynamically tunes the threshold of input similarity and manages cache based on the reuse opportunities.
FoggyCache~\cite{foggycache} is more general and can be applied to both image and audio inputs.
It designs adaptive LSH and homogenized KNN algorithms to address practical challenges in inference caching.
Instead of caching the final inference results, DeepCache~\cite{deepcache} stores the intermediate feature maps to achieve more granular reuse.
For object recognition, Glimpse~\cite{glimpse} maintains a cache of video frames on mobile devices.
It uses cached results to perform on-device object tracking and sends only trigger frames to the server with offloaded recognition model.

\textbf{Approaches tailored for specific pipelines.}
Focus~\cite{focus} is designed for querying detected objects in a video database and uses compressed CNN to index possible object classes at ingest stage and reduces the query latency by clustering similar objects.
Blazeit~\cite{blazeit} develops neural networks-based methods to optimize approximate aggregation queries of detected objects in video databases.
Focusing on object detection in video streams, Chameleon~\cite{chameleon} proposes to adaptively select a suitable pipeline configuration including the resolution and frame rate of videos, backbone neural networks for inference, etc.
Elf~\cite{elf} is designed for mobile video analytics where the input data is pre-processed by a lightweight on-device model and then offloaded in parallel to multiple servers with the same subsequent inference functionality.

Our proposed input-filtering framework unifies the frame filtering and inference caching approaches.
And we complement existing work in theoretical analysis and flexible supports for more input modalities and deployments.

\section{Conclusion}
In this paper, we study the input filtering problem and provide theoretical results on complexity comparisons between the hypothesis families of inference models and their input filters.
We propose the first end-to-end learnable framework that unifies both SKIP and REUSE methods and supports multiple input modalities and deployments.
We design and implement an input filter system \textit{InFi} based on our framework.
Comprehensive evaluations confirm our proven results and show that \textit{InFi} has wider applicability and outperforms strong baselines on accuracy and efficiency.


%



\ifCLASSOPTIONcompsoc
  \section*{Acknowledgments}
\else
  \section*{Acknowledgment}
\fi

This research was supported by the National Key R\&D Program of China 2021YFB2900103, China National Natural Science Foundation with No. 61932016, No. 62132018.
This work is partially sponsored by CAAI-Huawei MindSpore Open Fund and ``the Fundamental Research Funds for the Central Universities'' WK2150110024.

\ifCLASSOPTIONcaptionsoff
  \newpage
\fi



\bibliographystyle{IEEEtran}
%

\bibliography{reference.bib}




%


\begin{IEEEbiography}[{
    \includegraphics[width=1in,height=1.25in,clip,keepaspectratio]{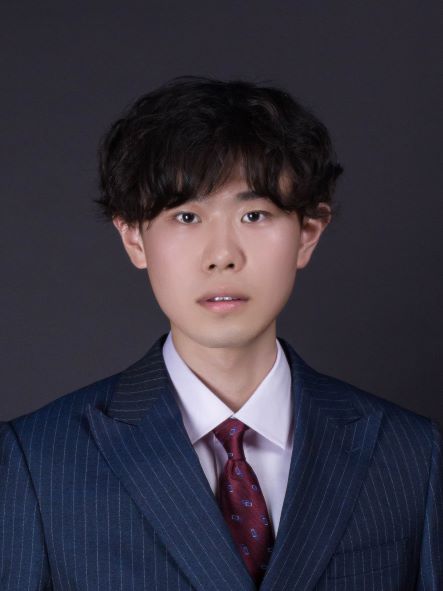}
}]{Mu Yuan}
is a Ph.D. candidate at the School of Computer Science and Technology, University of Science and Technology of China (USTC).
He received a bachelor's degree in computer science and technology from USTC.
His research interests include model inference and network systems.
\end{IEEEbiography}

\begin{IEEEbiography}[{
\includegraphics[width=1in,height=1.25in,clip,keepaspectratio]{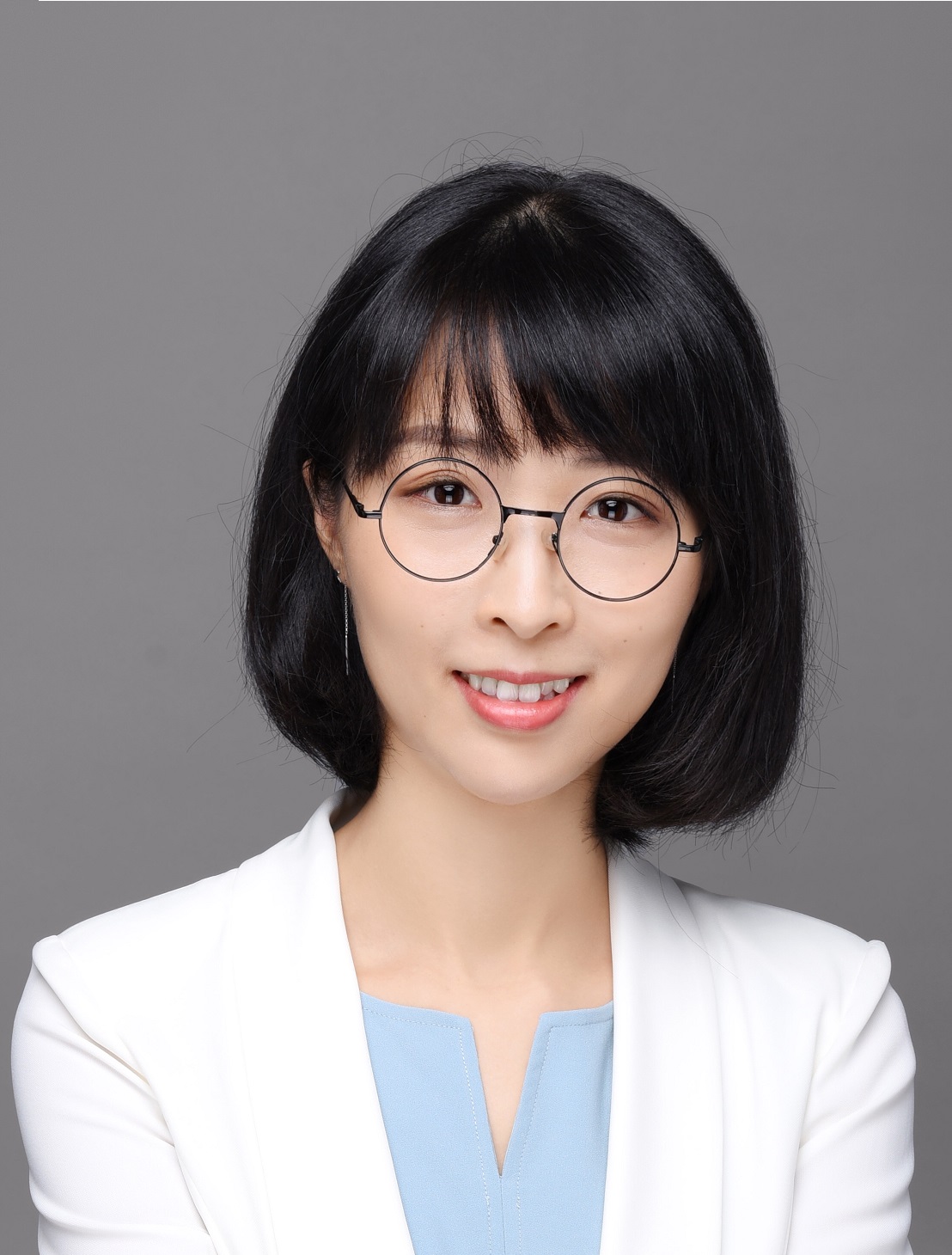}
}]{Lan Zhang}
is currently a Professor at the School of Computer Science and Technology, University of Science and Technology of China. She received her Ph.D degree and Bachelor degree from Tsinghua University, China. Her research interests include mobile computing, privacy protection,and data sharing and trading.
\end{IEEEbiography}

\begin{IEEEbiography}[{\includegraphics[width=1.0in,height=2.15in,clip,keepaspectratio]{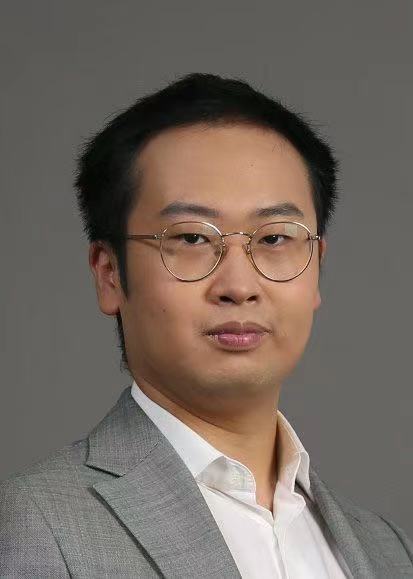}}]{Fengxiang He}
received his BSc in statistics from University of Science and Technology of China, MPhil and PhD in computer science from the University of Sydney. He is currently algorithm scientist at JD Explore Academy leading its trustworthy AI team. His research interest is in the theory and practice of trustworthy AI, including deep learning theory, privacy-preserving ML, algorithmic game theory, and decentralized learning. He publish in prominent venues, including ICML, NeurIPS, ICLR, CVPR, and ICCV. He is the area chair of prestigious conferences, AISTATS, BMVC, and ACML. He is the leading author of several standards.

\end{IEEEbiography}

\begin{IEEEbiography}[{\includegraphics[width=1.0in,height=2.15in,clip,keepaspectratio]{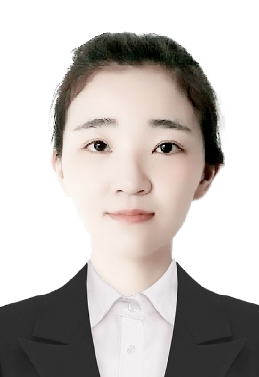}}]{Xueting Tong}
is a master's candidate at the Institute of Advanced Technology, University of Science and Technology of China. She received a bachelor's degree in Computer Science and technology from Henan University. Her research interests include model reasoning and optimization.
\end{IEEEbiography}

\begin{IEEEbiography}[{
\includegraphics[width=1in,height=1.25in,clip,keepaspectratio]{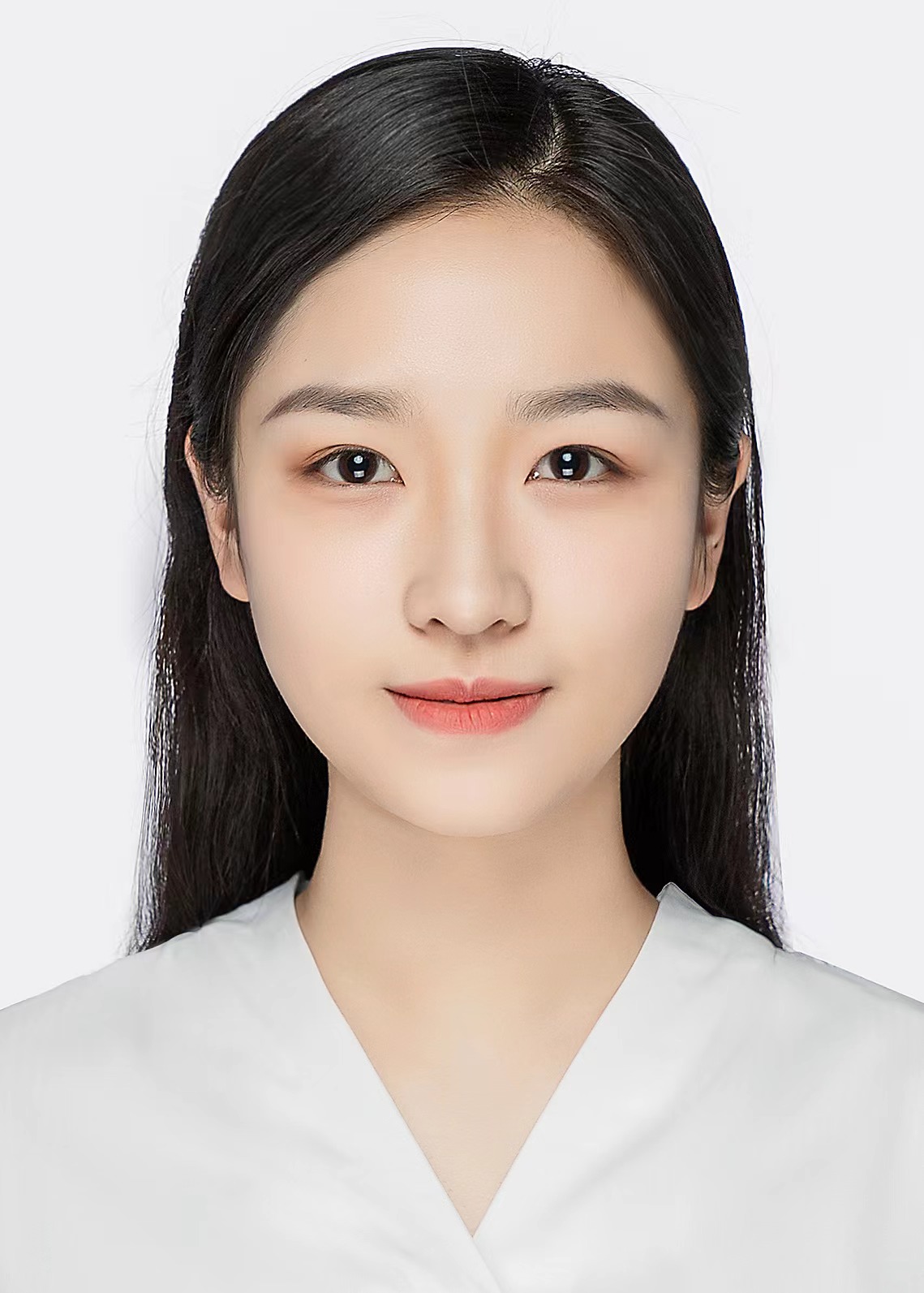}
}]{Miao-Hui Song}
is a master student at the School of Computer Science and Technology, University of Science and Technology of China. 
She received a bachelor's degree in computer science and technology from Chongqing University. 
Her research interests include active learning and data-labeling systems.
\end{IEEEbiography}

\begin{IEEEbiography}[{
\includegraphics[width=1in,height=1.25in,clip,keepaspectratio]{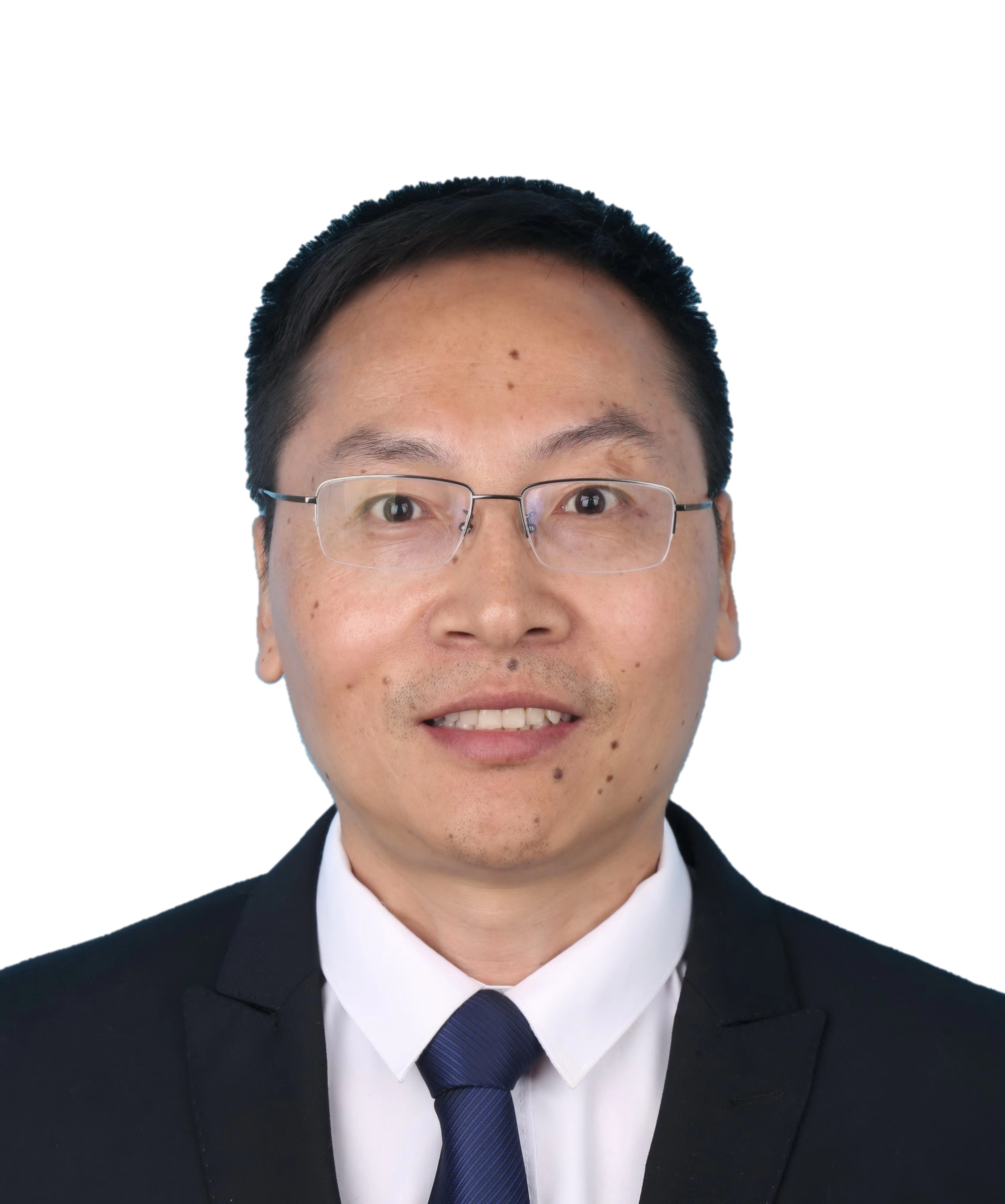}
}]{Zhengyuan Xu} received his B.S. and M.S. degrees from Tsinghua University, China, and Ph.D. degree from Stevens Institute of Technology, USA. He was a tenured full professor at University of California at Riverside and later at Tsinghua University before he joined University of Science and Technology of China (USTC). He was Founding Director of the multi-campus Center for Ubiquitous Communication by Light (UC-Light), University of California, and Founding Director of Wireless-Optical Communications Key Laboratory of Chinese Academy of Sciences. He was a distinguished expert and chief scientist of the National Key Basic Research Program of China. His research focuses on Petahertz communications, optical wireless communications, mobile networking, artificial intelligence, wireless big data, sensing, ranging and localization. He has published over 400 international journal and conference papers, and co-authored a book titled Visible Light Communications: Modulation and Signal Processing which has been selected by IEEE Series on Digital \& Mobile Communication and published by Wiley-IEEE Press. He has been on the Elsevier annual list of Most Cited Chinese Researchers since 2014. He has served as an Associate Editor for different IEEE/OSA journals and was a Founding Co-Chair of IEEE Workshop on Optical Wireless Communications in 2010.
\end{IEEEbiography}

\begin{IEEEbiography}[{\includegraphics[width=1.0in,height=2.15in,clip,keepaspectratio]{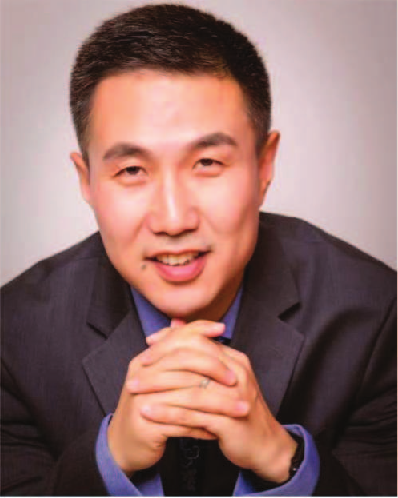}}]{Xiang-Yang Li}
 (Fellow, IEEE) is a professor and Executive Dean at School of Computer Science and Technology, USTC. He is an ACM Fellow (2019), IEEE fellow (2015), an ACM Distinguished Scientist (2014). He was a full professor at Computer Science Department of IIT and co-Chair of ACM China Council. Dr. Li received M.S. (2000) and Ph.D. (2001) degree at Department of Computer Science from University of Illinois at Urbana-Champaign. He received a Bachelor degree at Department of Computer Science from Tsinghua University, P.R. China, in 1995. His research interests include Artificial Intelligence of Things (AIOT), privacy and security of AIOT, and data sharing and trading.
\end{IEEEbiography}




\end{document}